\documentclass{article}

\usepackage{amsmath}
\usepackage{amsfonts}
\usepackage{dsfont}
\usepackage{amssymb}
\usepackage{amsmath}
\usepackage{amsthm}
\usepackage{dsfont}
\usepackage{xcolor}
\usepackage{url}
\usepackage{comment}
\usepackage{natbib}
\usepackage{float}
\usepackage{caption}
\usepackage{subcaption}
\usepackage{graphicx}
\usepackage{mathabx}
\usepackage{ulem}
\usepackage[margin=1in]{geometry}

\newcommand{\norm}[1]{\left\lvert\left\lvert#1\right\rvert\right\rvert}

\newcommand{\R}{\mathbb{R}}

\newcommand{\indicator}{\mathds{1}}
\newcommand{\iid}{\overset{\text{iid}}{\small \sim}}
\newcommand{\Lip}{\text{Lip}}

\newcommand{\Hh}{\mathcal{H}}
\newcommand{\Mm}{\mathcal{M}}

\newcommand{\Gg}{\mathcal{G}_{\textup{NN}}}
\newcommand{\E}{\mathbb{E}}

\newcommand{\argmin}{\operatornamewithlimits{\text{argmin}}}

\newcommand{\inj}{\text{inj}}

\newtheorem{theorem}{Theorem}
\newtheorem{proposition}{Proposition}
\newtheorem*{theorem*}{Theorem}
\newtheorem*{proposition*}{Proposition}

\newtheorem{lemma}{Lemma}
\newtheorem{lemma*}{Lemma}

\newtheorem*{corollary*}{Corollary}
\newtheorem{definition}{Definition}
\newtheorem*{definition*}{Definition}
\newtheorem{assumption}{Assumption}

\usepackage[utf8]{inputenc} % allow utf-8 input
\usepackage[T1]{fontenc}    % use 8-bit T1 fonts
\usepackage{hyperref}       % hyperlinks
\usepackage{url}            % simple URL typesetting
\usepackage{booktabs}       % professional-quality tables
\usepackage{amsfonts}       % blackboard math symbols
\usepackage{nicefrac}       % compact symbols for 1/2, etc.
\usepackage{microtype}      % microtypography
\usepackage{xcolor}         % colors

\title{On Deep Generative Models for Approximation and Estimation of Distributions on Manifolds }

\author{Biraj Dahal$^*$, Alex Havrilla$^*$, Minshuo Chen, Tuo Zhao, Wenjing Liao}

\begin{document}

\maketitle

\def\thefootnote{*}\footnotetext{These authors contributed equally to this work. Biraj Dahal, Alex Havrilla, Wenjing Liao are affiliated with the School of Mathematics at Georgia Institute of Technology. Minshuo Chen is affiliated with Department of Electrical and Computer Engineering at Princeton University. Tuo Zhao is affiliated with the School of Industrial and Systems Engineering (ISyE) at Georgia Institute of Technology. Email: \{bdahal,ahavrilla3,tourzhao,wliao60\}@gatech.edu,  mc0750@princeton.edu. This research is partially supported by NSF DMS 2012652 and NSF CAREER 2145167. }\def\thefootnote{\arabic{footnote}}

\begin{abstract}
  Generative networks have experienced great empirical successes in distribution learning. Many existing experiments have demonstrated that generative networks can generate high-dimensional complex data from a low-dimensional easy-to-sample distribution. However, this phenomenon can not be justified by existing theories. The widely held manifold hypothesis speculates that real-world data sets, such as natural images and signals, exhibit low-dimensional geometric structures. In this paper, we take such low-dimensional data structures into consideration by assuming that data distributions are supported on a low-dimensional manifold. We prove statistical guarantees of generative networks under the Wasserstein-1 loss. We show that the Wasserstein-1 loss converges to zero at a fast rate depending on the intrinsic dimension instead of the ambient data dimension. Our theory leverages the low-dimensional geometric structures in data sets and justifies the practical power of generative networks. We require no smoothness assumptions on the data distribution which is desirable in practice.
\end{abstract}

\section{Introduction}
Deep generative models, such as generative adversarial networks (GANs) \citep{goodfellow2014generative,arjovsky2017wasserstein} and variational autoencoder \citep{kingma2013auto,mohamed2014stochastic}, utilize neural networks to generate new samples which follow the same distribution as the training data.
They have been successful in many applications including producing photorealistic images, improving astronomical images, and modding video games \citep{reed2016generative, ledig2017photo, schawinski2017generative, brock2018large, volz2018evolving,radford2015unsupervised,salimans2016improved}.

To estimate a data distribution $Q$, generative models solve the following optimization problem
\begin{align}\label{eq:population}
 \textstyle \min_{g_\theta \in \mathcal{G}}~ {\tt discrepancy}((g_{\theta})_{\sharp}\rho, Q),
\end{align}
where $\rho$ is an easy-to-sample distribution, $\mathcal{G}$ is a class of generating functions, ${\tt discrepancy}$ is some distance function between distributions, and $(g_\theta)_\sharp \rho$ denotes the pushforward measure of $\rho$ under $g_\theta$. In particular, when we obtain a sample $z$ from $\rho$, we let $g_\theta(z)$ be the generated sample, whose distribution follows $(g_\theta)_\sharp \rho$.

{ There are many choices of the {\tt discrepancy} function in literature among which Wasserstein distance attracts much attention. The so-called Wasserstein generative models \citep{arjovsky2017wasserstein} consider the Wasserstein-1 distance defined as
\begin{align}\label{eq:wasserstein-1}
 \textstyle   W_1(\mu,\nu) = \sup\limits_{f \in {\rm Lip}_1(\R^D)} \mathbb{E}_{X \sim \mu} [f(X)] - \mathbb{E}_{Y \sim \nu} [f(Y)],
\end{align}
where $\mu, \nu$ are two distributions and ${\rm Lip}_1(\R^D)$ consists of $1$-Lipschitz functions on $\R^D$. The formulation in \eqref{eq:wasserstein-1} is known as the Kantorovich-Rubinstein dual form of Wasserstein-1 distance and can be viewed as an integral probability metric \citep{IPMMuller}.
}

{In deep generative models, the function class $\mathcal{G}$ is often parameterized by a deep neural network class $\mathcal{G}_{\rm NN}$. Functions in $\mathcal{G}_{\rm NN}$ can be written in the following compositional form
\begin{align}\label{nnform}
 \textstyle   g_{\theta}(x) = W_{L}\cdot \sigma(W_{L-1}\ldots \sigma(W_{1}x + b_1) + \ldots + b_{L-1}) +b_L,
\end{align} 
where the $W_i$'s and $b_i$'s are weight matrices and intercepts/biases of corresponding dimensions, respectively, and $\sigma$ is ReLU activation applied entry-wise: $\sigma(a) = \max(a, 0)$. Here $\theta=\{W_i,b_i\}_{i=1}^L$ denotes the set of parameters.
}

{Solving \eqref{eq:population} is prohibitive in practice, as we only have access to a finite collection of samples, $X_1, \dots, X_n \iid Q$. Replacing $Q$ by its empirical counterpart $Q_n = \frac{1}{n}\sum_{i=1}^n\delta_{X_i}$, we end up with
\begin{align}\label{emprisk}
  \textstyle  \hat{g}_n = \argmin\limits_{g_{\theta} \in \Gg}W_1((g_{\theta})_{\sharp}\rho, Q_n).
\end{align}}

Note that \eqref{emprisk} is also known as  training deep generative models under the Wasserstein loss in existing deep learning literature \citep{frogner2015learning,genevay2018learning}. It has exhibited remarkable ability in learning complex distributions in high dimensions, even though existing theories cannot fully explain such empirical successes. In literature, statistical theories of deep generative models have been studied in \citet{arora2017generalization, zhang2017discrimination, jiang2018computation,bai2018approximability, liang2017well, liang2018well,uppal2019nonparametric,CLZZ,lu2020universal,Block2021ANE,luise2020generalization,schreuder2021statistical}. Due to the well-known curse of dimensionality, the sample complexity in \citet{liang2017well,uppal2019nonparametric,CLZZ,lu2020universal} grows exponentially with respect to underlying the data dimension. 
%Such theories can not justify the empirical success of generative networks for high-dimensional data, since the required sample size given by these theories is far beyond the number of training samples used in applications. 
For example, the CIFAR-10 dataset consists of $32 \times 32$ RGB images. Roughly speaking, to learn this data distribution with accuracy $\epsilon$, the sample size is required to be $\epsilon^{-D}$ where $D = 32\times 32 \times 3 = 3072$ is the data dimension. Setting $\epsilon = 0.1$ requires $10^{3072}$ samples. However, GANs have been successful with $60,000$ training samples \citep{goodfellow2014generative}. 

{A common belief to explain the aforementioned gap between theory and practice is that practical data sets exhibit low-dimensional intrinsic structures.} 
% In machine learning, it is a common belief that real-world data sets exhibit low-dimensional structures, due to rich local regularities, global symmetries, or repetitive patterns.
For example, many
image patches are generated from the same pattern by some transformations, such as
rotation, translation, and skeleton. Such a generating mechanism induces a small number of intrinsic
parameters. It is plausible to model these data as samples near a low dimensional
manifold \citep{tenenbaum2000global,roweis2000nonlinear,peyre2009manifold,coifman2005geometric}.  

To justify that deep generative models can adapt to low-dimensional structures in data sets, this paper focuses (from a theoretical perspective) on the following fundamental questions of both distribution approximation and estimation:

\begin{description}
\item[Q1:]  {\it Can deep generative models approximate a distribution on a low-dimensional manifold by representing it as the pushforward measure of a low-dimensional easy-to-sample distribution?}

\item[Q2:] {\it If the representation in \textbf{Q1} can be learned by deep generative models, what is the statistical rate of convergence in terms of the sample size $n$?}
\end{description}

This paper provides positive answers to these questions.
We consider data distributions supported on a $d$-dimensional compact Riemannian manifold $\mathcal{M}$ isometrically embedded in ${\mathbb{R}}^{D}$. %The manifold is allowed to have multiple charts. We consider generative network models under the Wasserstein-1 loss.
The easy-to-sample distribution $\rho$ is uniform on $(0,1)^{d+1}$. To answer {\bf Q1}, our Theorem \ref{thm:approx} proves that deep generative models are capable of approximating a transportation map which maps the low-dimensional uniform distribution $\rho$ to a large class of data distributions on $\mathcal{M}$. To answer {\bf Q2}, our Theorem \ref{thm:stat} shows that the Wasserstein-1 loss in distribution learning 
converges to zero \textit{at a fast rate depending on the intrinsic dimension $d$} instead of the data dimension $D$.
In particular we prove that
    \begin{align*}
        \E W_1((\hat{g}_n)_\sharp \rho, Q) \leq C  n^{-\frac{1}{d+\delta}}
    \end{align*} for all $\delta > 0$ where $C$ is a constant independent of $n$ and $D$.

Our proof proceeds by constructing an  oracle transportation map $g^{*}$ such that $g^{*}_\sharp \rho = Q$.  This construction crucially relies on a cover of the manifold by geodesic balls, such that the data distribution $Q$ is decomposed as the sum of local distributions supported on these geodesic balls. Each local distribution is then transported onto lower dimensional sets in $\R^d$ from which we can apply optimal transport theory. We then argue that the oracle $g^*$ can be efficiently approximated by deep neural networks.

We make minimal assumptions on the network, only requiring that $g_\theta$ belongs to a neural network class (labelled $\Gg$) with size depending on some accuracy $\epsilon$. Further, we make minimal assumptions on the data distribution $Q$, only requiring that it admits a density that is upper and lower bounded. Standard technical assumptions are made on the manifold $\Mm$.

\section{Preliminaries}

We establish some notation and preliminaries on Riemannian geometry and optimal transport theory before presenting our proof.

\textbf{Notation.} For $x \in \R^d$, $\|x\|$ is the Euclidean norm, unless otherwise specified. $B_X(0,r)$ is the open ball of radius $r$ in the metric space $X$. If unspecified, we denote $B(0,r) = B_{\R^d}(0,r)$. For a function $f: \R^d \to \R^d$ and $A \subseteq \R^d$, $f^{-1}[A]$ denotes the pre-image of $A$ under $f$. $\partial$  denotes the differential operator. For $0 < \alpha \leq 1$, we denote by $C^{\alpha}$ the class of  H\"older continuous functions with H\"older index $\alpha$. $\| \cdot \|_\infty$ denotes the $\infty$ norm of a function, vector, or matrix (considered as a vector). For any positive integer $N \in \mathbb{N}$, we denote by $[N]$ the set $\{1, 2, \dots, N \}$.

\subsection{Riemannian Geometry}
Let $(\Mm, g)$ be a $d$-dimensional compact Riemannian manifold isometrically embedded in $\R^D$. Roughly speaking a manifold is a set which is locally Euclidean i.e. there exists a function $\phi$ continuously mapping a small patch on $\Mm$ into Euclidean space. This can be formalized with \textit{open sets} and \textit{charts}. At each point $x \in \Mm$ we have a \textit{tangent space} $T_x\Mm$ which, for a manifold embedded in $\R^D$, is the $d$-dimensional plane tangent to the manifold at $x$. We say $\Mm$ is Riemannian because it is equipped with a smooth metric $g_x: T_x\Mm  \times T_x\Mm \to \R$ (where $x$ is a basepoint) which can be thought of as a local inner product. We can define the Riemannian distance $d_\Mm : \Mm \times \Mm \rightarrow \R$ on $\Mm$ as
\begin{align*}
    d_{\Mm}(x,y) = \inf\{L(\gamma) | \gamma \textup{ is a } C^1(\Mm) \textup{ curve such that } \gamma(0) = x, \gamma(1) = y\},
\end{align*} i.e. the length of the shortest path or \textit{geodesic} connecting $x$ and $y$. 
%This {\color{red} This refers to?} will be complete on our compact manifold $\Mm$. 
An \textit{isometric embedding} of the $d$-dimensional $\Mm$ in $\R^D$ is an embedding that preserves the Riemannian metric of $\Mm$, including the Riemannian distance. For more rigorous statements, see the classic reference \citet{flaherty2013riemannian}.

 We next define the exponential map at a point $x \in \Mm$ going from the tangent space to the manifold.

\begin{definition}[Exponential map]
Let $x \in \Mm$. For all tangent vectors $v \in T_x \Mm$, there is a unique geodesic $\gamma$ that starts at $x$ with initial tangent vector $v$, i.e. $\gamma(0) = x$ and $\gamma'(0) = v$. The exponential map centered at $x$ is given by $\exp_x(v) = \gamma(1)$, for all $v \in T_x \Mm$.
\end{definition} 

The exponential map takes a vector $v$ on the tangent space $T_x \Mm$ as input. The output, $\exp_x(v)$, is the point on the manifold obtained by travelling along a (unit speed) geodesic curve that starts at $x$ and has initial direction $v$ (see Figure \ref{fig:exponential_map} for an example).

It is well known that for all $x \in \Mm$, there exists a radius $\delta$ such that the exponential map restricted to $B_{T_x\Mm}(0, \delta)$ is a diffeomorphism onto its image, i.e. it is a smooth map with smooth inverse. As the sufficiently small $\delta$-ball in the tangent space may vary for each $x \in \Mm$, we define the injectivity radius of $\Mm$ as the minimum $\delta$ over all $x \in \Mm$. 

\begin{definition}[Injectivity radius]
For all $x \in \Mm$, we define the injectivity radius at a point $\inj_{\Mm}(x) = \sup\{\delta > 0| \exp_{x}: B_{T_x \Mm}(0, \delta) \subseteq T_x \Mm \to \Mm \textup{ is a diffeomorphism}\}$. Then the injectivity radius of $\Mm$ is defined as
\begin{align*}
    \inj(\Mm) = \inf\{\inj_{\Mm}(x) | x \in \Mm\}.
\end{align*}
\end{definition} For any $x \in \Mm$, the exponential map restricted to a ball of radius $\inj_{\Mm}$ in $T_x\Mm$ is a well-defined diffeomorphism. Within the injectivity radius, the exponential map is a diffeomorphism between the tangent space and a patch of $\Mm$, with $\exp^{-1}$ denoting the inverse. Controlling a quantity called reach allows us to lower bound the manifold's injectivity radius.

\begin{definition}[Reach\citep{Federer}]
    The reach $\tau$ of a manifold $\Mm$ is defined as the quantity
    \begin{align*}
        \tau = \inf\{r > 0 : \exists x\neq y \in \Mm, v \in \R^D \textup{ such that } r =\|x- v\| = \|y- v\| = \inf_{z \in \Mm} \| z - v \| \}.
    \end{align*}
\end{definition}

Intuitively, if the distance of a point $x$ to $\Mm$ is smaller than the reach, then there is a unique point in $\Mm$ that is closest to $x$. However, if the distance between $x$ and $\Mm$ is larger than the reach, then there will no longer be a unique closest point to $x$ in $\Mm$. For example, the reach of a sphere is its radius. A manifold with large and small reach is illustrated in Figure \ref{fig:reach}. The reach gives us control over the injectivity radius $\inj({\Mm})$; in particular, we know $\inj({\Mm}) \geq \pi\tau$ (see \cite{AL} for proof). 

\begin{figure}
    \begin{minipage}[c]{0.5\textwidth}
    \centering
        \includegraphics[width=0.63\textwidth]{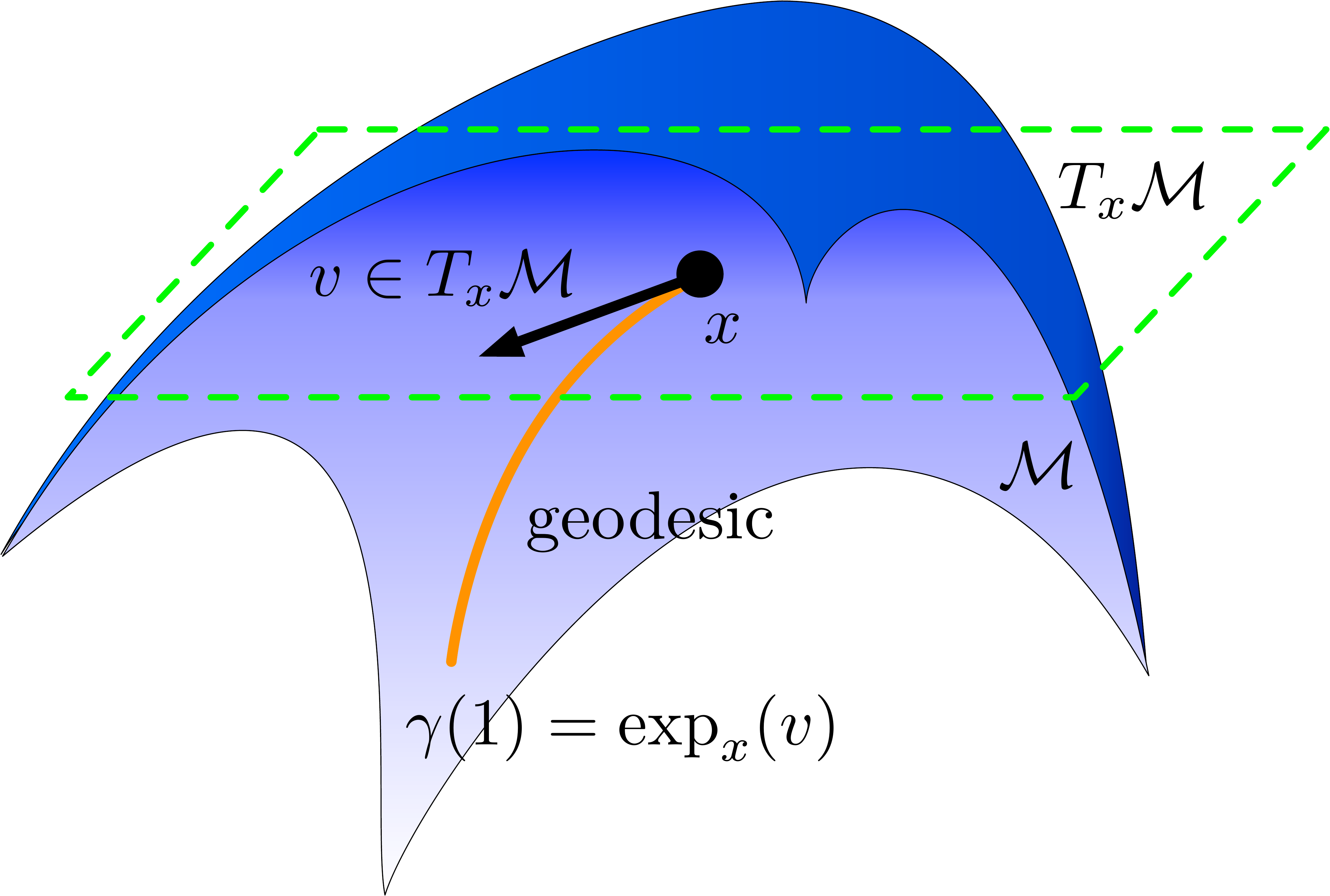}
        \caption{Exponential map on $\Mm$.}
        \label{fig:exponential_map}
    \end{minipage}
    \hfill
    \begin{minipage}[c]{0.52\textwidth}
        \centering
        \includegraphics[width=0.6\textwidth]{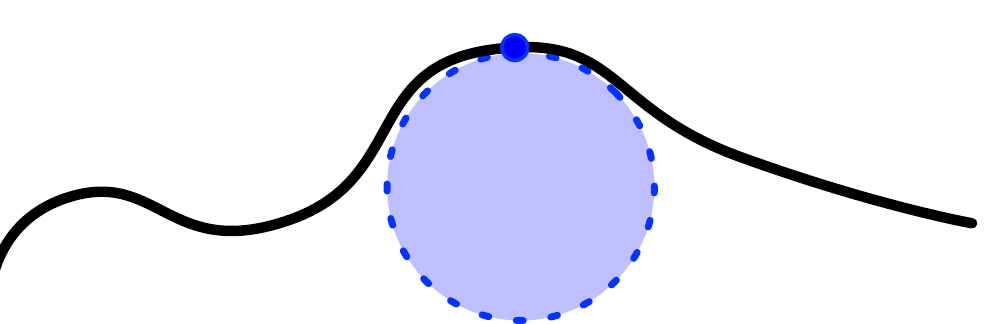}
        \\
        \vspace{0.2cm}
        \includegraphics[width=0.6\textwidth]{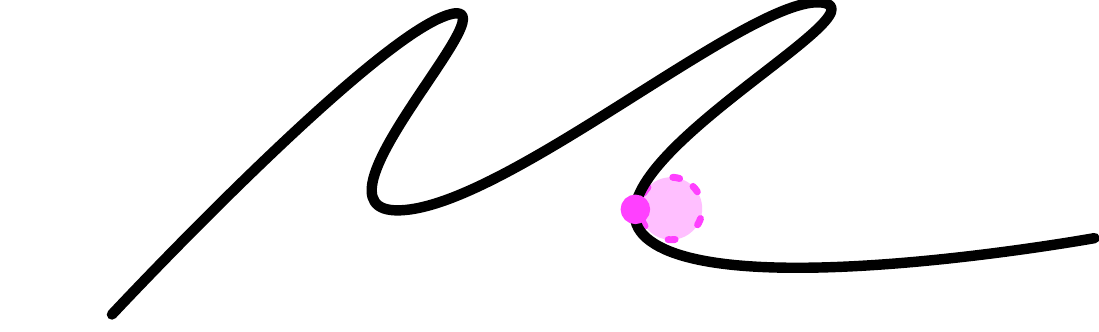}
        \caption{Manifolds with large and small reach.}
        \label{fig:reach}
    \end{minipage}
\end{figure}

\subsection{Optimal Transport Theory}
Let $\mu, \nu$ be absolutely continuous measures on sets $X, Y \subseteq \R^d$. We say a function $f: X \to Y$ \textit{transports} $\mu$ onto $\nu$ if $f_\sharp \mu = \nu$. In words, for all measurable sets $A$ we have 
\begin{align*}
    \nu(A) = f_\sharp \mu(A) = \mu \left(f^{-1}(A) \right),
\end{align*}
where $f^{-1}(A)$ is the pre-image of $A$ under $f$.
    Optimal transport studies the maps taking source measures $\mu$ on $X$ to target measures $\nu$ on $Y$ which also minimize a cost $c:X \times Y \to \R_{\geq 0}$ among all such transport maps. However the results are largely restricted to transport between measures on the same dimensional Euclidean space. In this paper, we will make use of the main theorem in \cite{caffarelli1992regularity}, in the form presented in \cite{villani2008optimal}.  %\url{https://www.researchgate.net/publication/243073627_The_Regularity_of_Mappings_with_a_Convex_Potential} }

\begin{proposition}
\label{proposition:OT}
Let $c(x,y) = \|x-y\|^2$ in $\R^d \times \R^d$ and let $\Omega_1, \Omega_2$ be nonempty, connected, bounded, open subsets of $\R^d$. Let $f_1, f_2$ be probability densities on $\Omega_1$ and $\Omega_2$ respectively, with $f_1, f_2$ bounded from above and below. Assume further that $\Omega_2$ is convex. Then there exists a unique optimal transport map $T : \Omega_1 \rightarrow \Omega_2$ for the associated probability measures $\mu(dx) = f_1(x) \, dx$ and $\nu(dy) = f_2(y) \, dy$, and the cost $c$. Furthermore, we have that $T \in C^{\alpha}(\Omega_1)$ for some $\alpha \in (0, 1)$. 

\end{proposition} 

This proposition allows to produce H\"older transport maps which can be further approximated with neural networks with size depending on a given accuracy.

To connect optimal transport and Riemannian manifolds, we first define the \textit{volume measure} on a manifold $\Mm$ and establish integration on $\Mm$. 

\begin{definition}[Volume measure]
	Let $\Mm$ be a compact $d$-dimensional Riemannian manifold. We define
	the volume measure $\mu_{\Mm}$ on $\Mm$ as the restriction of the $d$-dimensional Hausdorff measure $\mathcal{H}^d$.
\end{definition} 

A definition for the restriction of the Hausdorff measure can be found in \cite{Federer}.

We say that the distribution $Q$ has density $q$ if the Radon-Nikodym derivative of $Q$ with respect to $\mu_{\Mm}$ is $q$.
According to \cite{EG}), for any continuous function $f : \Mm \to \R$ supported within the image of the ball $B_{T_x \Mm}(0,\epsilon)$ under the exponential map for $\epsilon <\inj(\Mm)$, we have
\begin{align}\label{localdensity}
\int f \, d	Q = \int (fq) \, d	\mu_{\Mm} = \int_{B_{T_x\Mm}(0,\epsilon)} (fq) \circ \exp_x(v)\sqrt{\det g_{ij}^x(v)} \, dv.
\end{align}  Here $g_{ij}^x(v) = \langle \partial \exp_x(v)[e_i], \partial \exp_x(v)[e_j]\rangle$ with $(e_1,...,e_d)$ an orthonormal basis of $T_x\Mm$.

\section{Main Results}
We will present our main results in this section, including an approximation theory for a large class of distributions on a Riemannian manifold (Theorem \ref{thm:approx}), and a statistical estimation theory of deep generative networks for distribution learning (Theorem \ref{thm:stat}).

We make some regularity assumptions on a manifold $\Mm$ and assume the target data distribution $Q$ is supported on $\Mm$. The easy-to-sample distribution $\rho$ is taken to be uniform on $ (0,1)^{d+1}$.

\begin{assumption}\label{assumptionmanifold} $\Mm$ is a $d$-dimensional compact Riemannian manifold isometrically embedded in ambient dimension $\R^D$. Via compactness, $\Mm$ is bounded: there exists $M > 0$ such that $\|x\|_\infty \leq M$ ,  $\forall x \in \Mm$. Further suppose $\Mm$ has a positive reach $\tau > 0$.
\end{assumption}

\begin{assumption}\label{assumptiondistribution}  $Q$ is supported on $\Mm$ and has a density $q$ with respect to the volume measure on $\Mm$. Further we assume boundedness of $q$ i.e. there exists some constants $c, C > 0$ such that $c \leq q \leq C$.
\end{assumption}

To justify the representation power of feedforward ReLU networks for learning the target distribution $Q$, we explicitly construct a neural network generator class, such that a neural network function in this generator class can pushfoward $\rho$ to a good approximation of $Q$.

Consider the following generator class $\Gg$
\begin{align*}
\Gg(L,p,\kappa) = \{&g=[g_1,...,g_{D}]: \R^{d + 1} \to \R^{D} | g_j \textup{ in form } (\ref{nnform}) \text{ with at most } L \textup{ layers} \\
& \textup{ and max width } p, \text{ while } \norm{W_i}_{\infty} \leq \kappa, \norm{b_i}_{\infty} \leq \kappa \textup{ for all } i \in [L], j \in [D] \},
\end{align*} 
 
where $\|\cdot\|_{\infty}$ is the maximum magnitude in a matrix or vector. The width of a neural network is the largest dimension (i.e. number of rows/columns) among the $W_i$'s and the $b_i$'s.

\begin{theorem}[Approximation Power of Deep Generative Models] \label{thm:approx}
    %Let $\Mm$ be a $d$-dimensional compact Riemannian manifold satisfying assumption 1. Let $Q$ be a distribution on $\Mm$ satisfying assumption 2.
    Suppose $\Mm$ and $Q$ satisfy Assumptions \ref{assumptionmanifold} and \ref{assumptiondistribution} respectively.
    The easy-to-sample distribution $\rho$ is taken to be uniform on $ (0,1)^{d+1}$. Then there exists a constant $0 < \alpha < 1$
    (independent of $D$) such that for any $0 < \epsilon < 1$, there exists a $g_\theta \in \Gg(L, p, \kappa)$ with parameters
\begin{align*}
    L = &O\left(\log\left(\frac{1}{\epsilon}\right)\right),\hspace{0.1in} p = O\left(D\epsilon^{-\frac{d}{\alpha}}\right), \hspace{0.1in}  \kappa = M
\end{align*}
that satisfies
    \begin{align*}
        W_1((g_{\theta})_\sharp \rho, Q) < \epsilon.
    \end{align*}
\end{theorem}

Theorem \ref{thm:approx} demonstrates the representation power of deep neural networks for distributions $Q$ on $\Mm$, which answers Question {\bf Q1}. For a given accuracy $\epsilon$, there exists a neural network $g_{\theta}$ which pushes the uniform distribution on $(0,1)^
{d+1}$ forward to a good approximation of $Q$ with accuracy $\epsilon$. The network size is exponential in the intrinsic dimension $d$. A proof sketch of Theorem \ref{thm:approx} is given in Section \ref{subsection:appxproof}.

We next present a statistical estimation theory to answer Question {\bf Q2}. 

\begin{theorem}[Statistical Guarantees of Deep Wasserstein Learning] \label{thm:stat} %Let $\Mm$ be a $d$-dimensional compact Riemannian manifold satisfying assumption 1. Let $Q$ be a distribution on $\Mm$ satisfying assumption 2. 
Suppose $\Mm$ and $Q$ satisfy Assumption \ref{assumptionmanifold} and \ref{assumptiondistribution} respectively.
The easy-to-sample distribution $\rho$ is taken to be uniform on $ (0,1)^{d+1}$.
%Let $\rho$ be uniform on $\Zz = (0,1)^{d+1}$. 
Let $n$ be the number of samples of $X_i \sim Q$. Choose any $\delta > 0$. Set $\epsilon = n^{-\frac{1}{d+\delta}}$ in Theorem \ref{thm:approx} so that the network class $\Gg(L,p,\kappa)$ has parameters
\begin{align*}
    L = &O\left(\log\left(n^{\frac{1}{d+\delta}}\right)\right),\hspace{0.1in} p = O\left(Dn^{\frac{d}{\alpha(d+\delta)}}\right), \hspace{0.1in} \kappa = M.
\end{align*}
 Then the empirical risk minimizer $\hat{g}_n$ given by \eqref{emprisk} has rate
    \begin{align*}
        \E W_1((\hat{g}_n)_\sharp \rho, Q) \leq C n^{-\frac{1}{d+\delta}},
    \end{align*} where $C$ is a constant independent of $n$ and $D$.
\end{theorem} 

A proof sketch of Theorem \ref{thm:stat} is presented in Section \ref{subsection:statproof}. Additionally, this result can be easily extended to the noisy case. Suppose we are given $n$ noisy i.i.d. samples $\hat{X}_1,...,\hat{X}_n$ of the form $\hat{X}_i = X_i + \xi_i$, for $X_i \sim Q$ and $\xi_i$ distributed according to some noise distribution. The optimization in \eqref{emprisk} is performed with the noisy empirical distribution $\hat{Q}_n = \frac{1}{n}\sum_{i=1}^n \delta_{\hat{X}_i}$. Then the minimizer  $\hat{g}_n$ satisfies
\begin{align*}
    \E W_1((\hat{g}_n)_\sharp \rho, Q) \leq C n^{-\frac{1}{d+\delta}} + 2\sqrt{V_\xi},
\end{align*} where $V_\xi = \E \|\xi\|_2^2$ is the variance of the noise distribution. The proof in the noisy case is given in Section \ref{SubsecNoisy}.

\textbf{Comparison to Related Works}. To justify the practical power of generative networks, low-dimensional data structures are considered in  \cite{luise2020generalization,schreuder2021statistical,Block2021ANE,chae2021likelihood}. These works consider the generative models in \eqref{eq:population}.
They assume that the high-dimensional data are parametrized by low-dimensional latent parameters. Such assumptions correspond to the manifold model where the manifold is globally homeomorphic to  Euclidean space, i.e. the manifold has a single chart. 

In \cite{luise2020generalization}, the generative models are assumed to be continuously differentiable up to order $s$. By jointly training
of the generator and the latent distributions, they proved that the Sinkhorn divergence between the generated distribution and data distribution converges, depending on data intrinsic dimension. \cite{chae2021likelihood} and \cite{schreuder2021statistical} assume the special case where the manifold has a single chart. More recently, \cite{Block2021ANE} proposed to estimate the intrinsic dimension of data using the H\"older IPM between some empirical distributions of data. 
This theory is based on the statistical convergence of the empirical distribution to the data distribution. As an application to GANs, \cite[Theorem 23]{Block2021ANE} gives the statistical error while the approximation error is not studied. In these works, the single chart assumption is very strong while a general manifold can have multiple charts.

  Recently, \cite{yang2022capacity,huang2022error} showed that GANs can approximate any data distribution (in any dimension) by transforming an absolutely
continuous 1D distribution. The analysis in \cite{yang2022capacity,huang2022error} can be applied to the general manifold model. Their approach requires the GAN to memorize the empirical data distribution using ReLU networks. Thus it is not clear how the designed generator is capable of generating new samples that are different from the training data. In contrast, we explicitly construct an oracle transport map which transforms the low-dimensional easy-to-sample distribution to the data distribution. Our work provides insights about how distributions on a manifold can be approximated by a neural network pushforward of a low-dimensional easy-to-sample distribution without exactly memorizing the data. In comparison, the single-chart assumption in earlier works assumes that an oracle transport naturally exists. Our work is novel in the construction of the oracle transport for a general manifold with multiple charts, and the approximation theory by deep neural networks.

\section{Proof of Main Results}

\subsection{Proof of Approximation Theory in Theorem \ref{thm:approx}}
\label{subsection:appxproof}

To prove Theorem \ref{thm:approx}, we explicitly construct an oracle transport $g^*$ pushing $\rho$ onto $Q$, i.e. $g^*_\sharp \rho = Q$. Further this oracle will be piecewise $\alpha$-H\"older continuous for some $\alpha \in (0,1)$.

\begin{lemma}\label{lemma:oracleappx}
%Let $\Mm$ and $Q$ satisfy assumptions 1 and 2 respectively. 
Suppose $\Mm$ and $Q$ satisfy Assumption \ref{assumptionmanifold} and \ref{assumptiondistribution} respectively.
The easy-to-sample distribution $\rho$ is taken to be uniform on $(0,1)^{d+1}$.
Then there exists a function $g^*: (0,1)^{d+1} \to \Mm$ such that $Q = g^*_\sharp \rho$ where \begin{equation}g^*(x) =\textstyle \sum_{j=1}^J \indicator_{(\pi_{j-1}, \pi_j)}(x_1) g^*_j(x_{2:d+1})
\label{eqgstar}
\end{equation}
for some $\alpha$-H\"older ($0<\alpha<1$) continuous functions $g_1^*, \dots, g_J^*$ and some constants $0 = \pi_0 < \pi_1 < \dots < \pi_J = 1$.
\end{lemma} 

\label{proof:oracleappx}

\begin{proof}

We construct a transport map $g^*: (0,1)^{d+1} \to \Mm$ that can be approximated by neural networks. First, we decompose the manifold into overlapping geodesic balls. Next, we pull these local distributions on these balls back to tangent space, which produces $d$-dimensional tangent distributions. Then, we apply optimal theory on these tangent distributions to produce maps between the source distributions on $(0,1)^{d}$ to the appropriate local (geodesic ball) distributions on the manifold. Finally, we glue together these local maps with indicators functions and a uniform random sample from $(0,1)$. We proceed with the first step of decomposing the manifold.

\textbf{Step 1: Overlapping ball decomposition.} Recall that $\Mm$ is a compact manifold with reach $\tau >0$. Then the injectivity radius of $\Mm$ is greater or equal to $\pi \tau$ (\citet{aamari2019estimating}). Set $r = \frac{\pi \tau}{2}$. For each $c \in \Mm$, define an open set $U_c = \exp_c(B_{T_c \Mm}(0, r)) \subseteq \Mm$. Since the collection $\{ U_c : x \in \Mm \}$ forms an open cover of $\Mm$ (in $\R^D$), by the compactness of $\Mm$ we can extract a finite subcover which we denote as $\{ U_{c_j} \}_{j=1}^J$. For convenience, we will write $U_j = U_{c_j}$. 

\textbf{Step 2: Defining local lower-dimensional distributions.} On each $U_j$, we define a local distribution $Q_j$ with density $q_j$ via $$q_j(x) = \frac{q(x)}{Q(U_j)}\indicator_{U_j}(x).$$ Set $K(x) = \sum_{j=1}^J \indicator_{U_j}(x)$ as the number of balls $U_j$ containing $x$. Note $1 \leq K(x) \leq J$ for all $x \in \mathcal{M}$. Now define the distribution $\overline{Q}_j$ with density $\overline{q}_j$ given by  

\[ \overline{q}_j(x) = \frac{\frac{1}{K(x)} q_j(x) \indicator_{U_j}(x)}{\int_{U_j}\frac{1}{K(x)} q_j(x)d\mathcal{H}} . \]

Write $K_j = \int_{U_j}\frac{1}{K(x)}q_j(x)d\mathcal{H}$ as the normalizing constant. Define $\tilde{q}_j(v) = (\overline{q}_j \circ \exp_{c_j})(v)\sqrt{\det g_{kl}^{c_j}(v)}$ where $g_{kl}^{c_j}$ is the Riemannian at $c_j$. This quantity can be thought of as the Jacobian of the exponential map, denoted by $|J_{\exp_{c_j}}(v)|$ in the following step. Then $\tilde{q}_j$ is a density on $\tilde{U}_j = \exp^{-1}_{c_j}(U_j)$, which is a ball of radius $\frac{\pi \tau}{2}$ since
\begin{align*}
   1 = \int_{U_j}\overline{q}_j(x)d\Hh = \int_{\tilde{U}_j}\sqrt{\det g_{kl}^{c_j}(v)} \overline{q}_j(\exp_{c_j}(v))dv = \int_{\tilde{U}_j} \tilde{q}_j(v)dv
\end{align*}  

Let $\tilde{Q}_j$ be the distribution in $\R^d$ with density $\tilde{q}_j$. By construction, we can write \begin{equation}\overline{Q}_j = (\exp_{c_j})_\sharp \tilde{Q}_j.
\label{ProofLemma1EqComp1}
\end{equation} 

\textbf{Step 3: Constructing the local transport.}  We have that $\exp_{c_j}^{-1}$ is bi-Lipschitz on $U_j$ and hence its Jacobian is upper bounded. Since $|J_{\exp_{c_j}}(v)| = \frac{1}{|J_{\exp_{c_j}^{-1}}(x)|}$, we know that $|J_{\exp_{c_j}}|$  lower bounded. Since $q_j$ is lower bounded (away from $0$), this means $\tilde{q}_j$ is also lower bounded. Now the distribution $\tilde{\rho}_j$ supported on $\tilde{U}_j = B(0,\frac{\tau \pi}{2})$ fulfills the requirements for our optimal transport result: (1) Its density $\tilde{p}_j$ is lower and upper bounded; (2) The support $B(0,\frac{\tau \pi}{2})$ is convex. 
Taking our cost to be $c(x,y) = \frac{1}{2}\|x-y\|^2$ (i.e. squared Euclidean distance), via Proposition \ref{proposition:OT} we can find an optimal transport map $T_j$ such that
\begin{equation}
(T_j)_{\sharp}\rho_d = \tilde{Q}_j
    \label{ProofLemma1EqComp3}
\end{equation} where $\rho_d$ is uniformly distributed on $(0,1)^d$. Furthermore, $T_j \in C^{\alpha_j}$ for some $\alpha_j \in (0,1)$. Then we can construct a local transport onto $U_j$ via 
\begin{equation}g_j^* = \exp_{c_j} \circ T_j
\label{eq:gjstar}
\end{equation}
which pushes $\rho_d$ forward to $\overline{Q}_j$. Since $g_j^*$ is a composition of a Lipschitz map with an $\alpha_j$ H\"older continuous maps, it is hence $\alpha_j$ H\"older continuous. 

\begin{figure}[h!]
    \centering
    \includegraphics[height=1.25in]{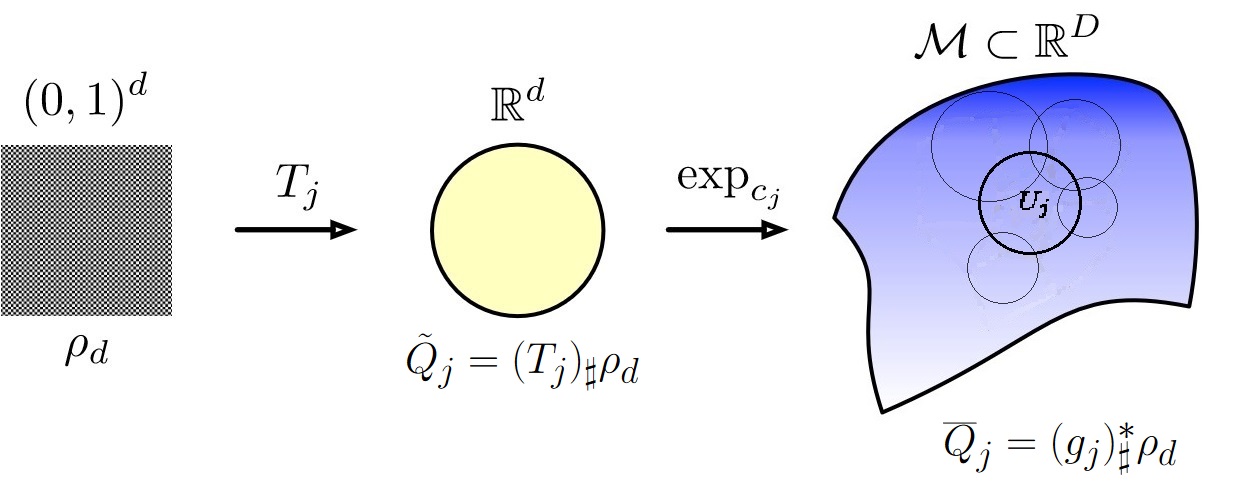}
    \caption{Local transport $g_j^*$ in \eqref{eq:gjstar} mapping $\rho_d$ on $(0,1)^d$ to a local distribution $\overline{Q}_j$ supported on $U_j$.}
    \label{fig:simple_flowchart}
\end{figure}

\textbf{Step 4: Assembling the global transport.} It remains to patch together the local distributions $\overline{Q}_j$ to form $Q$. Define $\eta_j = K_jQ(U_j)$. Notice 

\begin{align*}
    \sum_{j=1}^J \eta_j \overline{q}_j(x) &= \sum_{j=1}^J  K_jQ(U_j) \frac{\frac{1}{K(x)} q_j(x) \indicator_{U_j}(x)}{K_j} = \sum_{j=1}^J  Q(U_j) \frac{\frac{1}{K(x)} q(x) \indicator_{U_j}(x)}{Q(U_j)} \\
    &= \sum_{j=1}^J \frac{1}{K(x)} q(x) \indicator_{U_j}(x) = q(x)\frac{1}{K(x)}\sum_{j=1}^J \indicator_{U_j}(x) = q(x). 
\end{align*} Hence it must be that $\sum_{j=1}^J \eta_j = 1$. Set $\alpha = \min_{j \in [J]} \alpha_j$. We can now define the oracle $g^*$. Let $x \in (0,1)^{d+1}$. Write
\begin{align}
    g^*(x) =\textstyle \sum_{j=1}^J \indicator_{(\pi_{j-1}, \pi_j) }(x_1)g_j^*(x_{2:d+1}),
\label{eq:gstar}
\end{align} where $x_1$ is the first coordinate and $x_{2:d+1}$ are the remaining coordinates with $\pi_j = \sum_{i=1}^{j-1}\eta_i$.  Let $Z \sim \rho$. Then $g(Z) \sim Q$. We see this as follows. For $A \subseteq \mathcal{M}$ we can compute 
\begin{align*}
    \mathbb{P}(g^*(Z) \in A) &= \sum_{j=1}^J\mathbb{P}(\pi_{j-1} < Z_1 < \pi_j) \mathbb{P}(g^*_j(Z_{2:d+1}) \in A \cap U_j) = \sum_{j=1}^J \eta_j \overline{Q}_j(A \cap U_j) \\
    &= \sum_{j=1}^J \eta_j \int_{A}\overline{q}_j(x) d\mathcal{H} = \int_A \sum_{j=1}^J \eta_j \overline{q}_j(x)d\mathcal{H} = \int_A q(x)d\mathcal{H} = Q(A)
\end{align*} which completes the proof.
\end{proof}

We have found an oracle $g^*$ which is piecewise H\"older continuous such that $g^*_\sharp \rho = Q$. We can design a neural network $g_\theta$ to approximate this oracle $g^*$. Now in order to minimize $W_1((g_\theta)_\sharp \rho, Q) = W_1((g_\theta)_\sharp \rho, g^*_\sharp \rho)$, we show it suffices to have $g_\theta$ approximate $g^*$ in $L^1(\rho)$.

\begin{lemma}
\label{lemma:w1}
Let $\mu$ be an absolutely continuous probability distribution on a set $Z \subseteq \R^d$, and let $f, g: Z \rightarrow \R^m$ be transport maps. Then $$W_1(f_\sharp \mu, g_\sharp \mu) \leq C\|f-g\|_{L^1(\mu)}$$ for some $C > 0$.
\end{lemma}
\begin{proof}[Proof of Lemma \ref{lemma:w1}]
The vector-valued functions $f$ and $g$ output  $m$-dimensional vectors. Note that  $||f-g||_{L^1(\mu)} = \sum_{i=1}^m \| f_i - g_i \|$ where $f_i$ and $g_i$ denote the $i$th component function of $f$ and $g$, respectively. Then we can compute
\begin{align*}
W_1(f_\sharp \mu, g_\sharp \mu)
    &= \sup_{\phi \in \Lip_1(\R^m)} \left| \int \phi(y) \, d(f_\sharp \mu) - \int \phi(y) \, d(g_\sharp \mu)  \right| \\
    &= \sup_{\phi \in \Lip_1(\R^m)} \left| \int \phi(f(x)) - \phi(g(x)) \, d\mu \right| \\
    &\leq \sup_{\phi \in \Lip_1(\R^m)} \int \left|\phi(f(x)) - \phi(g(x)) \right| \, d\mu  \\
    &\leq \int_Z \left\|f(x) - g(x) \right\|_2 \, d\mu  \\
    &\leq \int_Z C \|f(x)-g(x)\|_1 \, d\mu \\
    &= C \|f - g\|_{L^1(\mu)},
\end{align*} 
since $\phi$ is Lipschitz with constant $1$ and all norms are equivalent in finite dimensions. In particular, $C=1$ here.

\end{proof}

We now prove Theorem \ref{thm:approx}.

\begin{proof}[Proof of Theorem \ref{thm:approx}]

By Lemma \ref{lemma:oracleappx}, there exists a transformation $g^*(x) = \sum_{j=1}^J \indicator_{(\pi_{j-1}, \pi_j)}(x_1) g^*_j(x_{2:d+1})$ such that $g^* _\sharp \rho = Q$. %This proof focuses on constructing a ReLU network to approximate $g^*$ in $L^1(\rho)$ norm.
By Lemma \ref{lemma:w1}, it suffices to approximate $g^*$ with a neural network $g_\theta \in \Gg(L, p, \kappa)$ in $L^1$ norm, with a given accuracy $\epsilon > 0$.  Let $(g^*)^{(i)}$ denote the $i$th component of the vector valued function $g^*$. Then it suffices to approximate
\[ (g^*)^{(i)}(x) = \sum_{j=1}^J \indicator_{(\pi_{j-1}, \pi_j)}(x_1) (g^*_j)^{(i)}(x_{2:d+1}) \]
for each $1 \leq i \leq D$, where $(g^*_j)^{(i)}$ denotes the $i$th component of the function $g^*_j$.
We construct the approximation of $(g^*)^{(i)}$ by the function 
\begin{equation}
\label{eq:gtheta}
(g_\theta)^{(i)}(x) = \sum_{j=1}^J \tilde{\times}^{\delta_2} \left(\tilde{\indicator}^{\delta_1}_{(\pi_{j-1}, \pi_j)}(x_1), (g_{j, \theta}^{\delta_3})^{(i)}(x_{2:d+1}) \right),
\end{equation} 
where $\tilde{\times}^{\delta_2}$ is a ReLU network approximation to the multiplication operation with $\delta_2$ accuracy, $\tilde{\indicator}^{\delta_1}_{(\pi_{j-1}, \pi_j)}$ is a ReLU network approximation to the indicator function with $\delta_1$ accuracy, and $(g_{j, \theta }^{\delta_3})^{(i)}$ is a ReLU network approximation to $(g_j^*)^{(i)}$ with $\delta_3$ accuracy. We construct these using the approximation theory outlined in Appendix \ref{sec:appendixapprox}.

First, we obtain $\tilde{\indicator}^{\delta_1}_{(\pi_{j-1}, \pi_j)}$ via an application of Lemma \ref{lemma:indicator}. Next, we obtain $\tilde{\times}^{\delta_2}$ from an application of Lemma \ref{lemma:times}. Finally, we discuss $g_{j, \theta}^{\delta_3}$. Let $j \in [J]$. To approximate the H\"older function $g_j^*$, we use the following Lemma \ref{lemma:main} that is proved in Appendix \ref{sec:appendixapprox}. Similar approximation results can be found in \cite{shen2022optimal} and \cite{ohn2019smooth} as well. In Lemma \ref{lemma:main}, our approximation error is in $L^1$ norm and all weight parameters are upper bounded by a constant. In  comparison, the error in \cite{ohn2019smooth} is in $L^\infty$ norm and the weight parameter increases as $\epsilon$ decreases.

\begin{lemma}
\label{lemma:main}
Fix $M \geq 2$. Suppose $f \in C^\alpha([0,1]^d)$, $\alpha \in (0, 1]$, with $\|f\|_{L^\infty} < M$. Let $0 < \epsilon < 1$. Then there exists a function $\Phi$ implementable by a ReLU network such that $$\|f - \Phi\|_{L^1} < \epsilon.$$ The ReLU network has depth at most $c_1 \log\left(\frac{1}{\epsilon} \right)$, width at most $c_2 \epsilon^{-\frac{d}{\alpha}}$, and weights bounded by $M$ (where $c_1$ and $c_2$ are constants independent of $\epsilon$).
\end{lemma}

We can apply Lemma \ref{lemma:main} to $(g_j^*)^{(i)}$ for all $1 \leq j \leq J$ and $1 \leq i \leq D$, since they are all elements of $C^{\alpha}(0,1)^d$ and elements of $C^{\alpha}(0,1)^d$ can be extended to $C^{\alpha}[0,1]^d$. 
Thus there exists a neural network $(g_{j, \theta}^{\delta_3})^{(i)} \in \Gg(L, p, \kappa)$ with parameters given as above such that $$\|(g_j^*)^{(i)} - (g_{j, \theta}^{\delta_3})^{(i)}\|_{L^1} < \delta_3.$$

The goal is now to show the $L^1$ distance between $g_\theta$ (as defined in \eqref{eq:gtheta}) and $g^*$ is small. We compute
\begin{align*}
& \quad \|g^*-g_\theta\|_{L^1} \\
&= \sum_{i=1}^D \|(g^*)^{(i)} - (g_\theta)^{(i)}\|_{L^1} \\ 
    &= \sum_{i=1}^D \int_{(0,1)^{d+1}} \left|(g^*)^{(i)}(x) - (g_{\theta})^{(i)}(x)\right| \, dx \\
    &\leq \sum_{i=1}^D \sum_{j=1}^J \int_{(0,1)^{d+1}}\left| \tilde{\times}^{\delta_2} \left(\tilde{\indicator}^{\delta_1}_{(\pi_{j-1}, \pi_j)}(x_1), (g_{j, \theta}^{\delta_3})^{(i)}(x_{2:d+1}) \right) - \indicator_{(\pi_{j-1}, \pi_j)}(x_1) (g^*_j)^{(i)}(x_{2:d+1})\right| \, dx\\
    &\leq \sum_{i=1}^D \sum_{j=1}^J \int_{(0,1)^{d+1}}\left| \tilde{\times}^{\delta_2} \left(\tilde{\indicator}^{\delta_1}_{(\pi_{j-1}, \pi_j)}(x_1), (g_{j, \theta}^{\delta_3})^{(i)}(x_{2:d+1}) \right) - \tilde{\indicator}^{\delta_1}_{(\pi_{j-1}, \pi_j)}(x_1)(g_{j, \theta}^{\delta_3})^{(i)}(x_{2:d+1}) \right| \, dx\\
    &\quad+ \sum_{i=1}^D \sum_{j=1}^J \int_{(0,1)^{d+1}}\left|\tilde{\indicator}^{\delta_1}_{(\pi_{j-1}, \pi_j)}(x_1)(g_{j, \theta}^{\delta_3})^{(i)}(x_{2:d+1}) - \indicator_{(\pi_{j-1}, \pi_j)}(x_1)(g_{j, \theta}^{\delta_3})^{(i)}(x_{2:d+1}) \right| \, dx \\
    &\quad+ \sum_{i=1}^D \sum_{j=1}^J \int_{(0,1)^{d+1}}\left|\indicator_{(\pi_{j-1}, \pi_j)}(x_1)(g_{j, \theta}^{\delta_3})^{(i)}(x_{2:d+1}) - \indicator_{(\pi_{j-1}, \pi_j)}(x_1)(g_j^*)^{(i)}(x_{2:d+1}) \right| \, dx \\
    &= \sum_{i=1}^D \sum_{j=1}^J \left(\text{(I) + (II) + (III)}\right)
\end{align*} 

Each of the three terms are easily handled as follows.

\begin{enumerate}
    \item[(I)] By construction of $\tilde{\times}^{\delta_2}$ in Lemma \ref{lemma:times}, we have that 
    
    \begin{align*}
    \text{(I)}
        &= \int_{(0,1)^{d+1}} \left| \tilde{\times}^{\delta_2} \left(\tilde{\indicator}^{\delta_1}_{(\pi_{j-1}, \pi_j)}(x_1), (g_{j, \theta}^{\delta_3})^{(i)}(x_{2:d+1}) \right) - \tilde{\indicator}^{\delta_1}_{(\pi_{j-1}, \pi_j)}(x_1)(g_{j, \theta}^{\delta_3})^{(i)}(x_{2:d+1}) \right| \, dx \\
        &\leq \int_{(0,1)^{d+1}} \delta_2 \, dx = \delta_2.
    \end{align*}

    \item[(II)] By construction of $\tilde{\indicator}^{\delta_1}_{(\pi_{j-1}, \pi_j)}$ in Lemma \ref{lemma:indicator}, we have that 
    
    \begin{align*}
    \text{(II)}
        &= \int_{(0,1)^{d+1}} \left|\tilde{\indicator}^{\delta_1}_{(\pi_{j-1}, \pi_j)}(x_1)(g_{j, \theta}^{\delta_3})^{(i)}(x_{2:d+1}) - \indicator_{(\pi_{j-1}, \pi_j)}(x_1)(g_{j, \theta}^{\delta_3})^{(i)}(x_{2:d+1}) \right| \, dx \\
        &\leq \left\|(g_{j, \theta}^{\delta_3})^{(i)} \right\|_{\infty} \int_0^1 \left|\tilde{\indicator}^{\delta_1}_{(\pi_{j-1}, \pi_j)}(x_1) - \indicator_{(\pi_{j-1}, \pi_j)}(x_1)\right| \, dx \\
        &\leq  M \left\| \tilde{\indicator}^{\delta_1}_{(a, b)} -  \indicator_{(a, b)}\right\|_{L^1} \\
        &= M \delta_1.
    \end{align*}
    
    \item[(III)] By construction of $(g_{j, \theta}^{\delta_3})^{(i)}$ from Lemma \ref{lemma:main}, we have that 
    \begin{align*}
    \text{(III)}
        &= \int_{(0,1)^{d+1}} \left|\indicator_{(\pi_{j-1}, \pi_j)}(x_1)(g_{j, \theta}^{\delta_3})^{(i)}(x_{2:d+1}) - \indicator_{(\pi_{j-1}, \pi_j)}(x_1)(g_{j}^*)^{(i)}(x_{2:d+1}) \right| \, dx \\
        &= \|\indicator_{(\pi_{j-1}, \pi_j)}\|_\infty \int_{(0,1)^{d}} \left|(g_{j, \theta}^{\delta_3})^{(i)}(x) - (g_{j}^*)^{(i)}(x) \right| \, dx \\
        &= \left\| (g_{j, \theta}^{\delta_3})^{(i)} - (g_{j}^*)^{(i)} \right\|_{L^1} \\
        &\leq \delta_3.
    \end{align*}
\end{enumerate}

As a result, we have that
\[ \|g^* - g_{\theta}\|_{L^1} \leq \sum_{i=1}^D \sum_{j=1}^J \text{(I)} + \text{(II)} + \text{(III)} \leq \sum_{i=1}^D \sum_{j=1}^J \delta_2 + M\delta_1 + \delta_3 = DJ(M\delta_1 + \delta_2 + \delta_3). \]
By selecting $\delta_1 < \frac{\epsilon}{3DJM}$, $\delta_2 < \frac{\epsilon}{3DJ}$, and $\delta_3 < \frac{\epsilon}{3DJ}$, we obtain that $||g^* - g_{\theta}||_1 < \epsilon$. 

To complete the proof, we note that $g_\theta$ can be exactly represented by a neural network in $\Gg(L, p, \kappa)$ with parameters
\begin{align*}
    L = &O\left(\log\left(\frac{1}{\epsilon}\right)\right),\hspace{0.1in} p = O\left(D\epsilon^{-\frac{d}{\alpha}}\right), \hspace{0.1in}  \kappa = M.
\end{align*}
\end{proof}

\subsection{Proof of Statistical Estimation Theory in Theorem \ref{thm:stat}}
\label{subsection:statproof}

The proof of Theorem \ref{thm:stat} is facilitated by the common bias-variance inequality, presented here as a lemma.

\begin{lemma}
\label{lemma:biasvariance}
Under the same assumptions of Theorem \ref{thm:stat}, we have
\begin{align}
\label{eq:biasvariance}
    \E W_1((\hat{g}_n)_{\sharp}\rho, Q) \leq \inf_{g_{\theta} \in \Gg }W_1((g_{\theta})_{\sharp}\rho, Q) + 2 \E W_1(Q_n,Q)
\end{align} where $Q_n$ is the clean empirical distribution.
\end{lemma}

\begin{proof} Recalling the definition of $\hat{g}_n$ as the empirical risk minimizer, we compute
    \begin{align*}
\E W_1((\hat{g}_n)_{\sharp}\rho, Q)
    &\leq \E W_1((\hat{g}_n)_{\sharp}\rho, Q_n) + \E W_1(Q_n, Q) \\
    &= \E \inf_{g_\theta \in \Gg} W_1((g_\theta)_{\sharp}\rho, Q_n) + \E W_1(Q_n, Q)\\
    &\leq \E \inf_{g_\theta \in \Gg} W_1((g_\theta)_{\sharp}\rho, Q) + 2\E W_1(Q_n, Q) \\
\end{align*}
since $W_1((\hat{g}_n)_\sharp \rho, Q_n) = \inf_{g_\theta \in \mathcal{G}_{\rm NN}} W_1((g_\theta)_\sharp \rho, Q_n)$ from \eqref{emprisk}.

\end{proof}

In the right hand side of \eqref{eq:biasvariance}, the first term is the approximation error and the second term is the statistical error. This naturally decomposes the problem into two parts: one controlling the approximation error and the other controlling the statistical error. The bias term can be controlled via Theorem \ref{thm:approx}. To control convergence of the empirical distribution $Q_n$ to $Q$ we leverage the existing theory \citep{Weed2019SharpAA} to obtain the following lemma. 

\begin{lemma}
\label{lemma:statconverge}
Under the same assumptions of Theorem \ref{thm:stat}, for all $\delta > 0$, $\exists C_\delta > 0$ such that 
\begin{equation}
\label{eq:statrate}
\E \left[W_1(Q, Q_n) \right] \leq C_\delta n^{-\frac{1}{d + \delta}}.
\end{equation}
\end{lemma}
\begin{proof}[Proof of Lemma \ref{lemma:statconverge}]
Let $\delta > 0$. Consider the manifold $\Mm$ with the geodesic distance as a metric space. When \cite[Theorem 1]{Weed2019SharpAA} is applied to $\Mm$ with the geodesic distance, we have that $$\E \left[W^{\Mm}_1(Q, Q_n) \right] \leq  C_\delta n^{-\frac{1}{d + \delta}}$$ for some constant $C_\delta$ independent of $n$. Here, $W^\Mm_1$ is the $1$-Wasserstein distance on $\Mm$ with the geodesic distance. It suffices to show that $$W_1^{\R^D}(Q, Q_n) = W_1(Q, Q_n) \leq W_1^{\Mm}(Q, Q_n).$$ Let $\Lip_1(\R^D)$ and $\Lip_1(\Mm)$ denote the set of $1$-Lipschitz functions defined on $\Mm$ with respect to the Euclidean distance on $\R^D$ and geodesic distance on $\Mm$ respectively. But note that $\Lip_1(\R^D) \subseteq \Lip_1(\Mm)$ because for any $f \in \Lip_1(\R^D)$ we have
\begin{align*}
     \frac{|f(x)-f(y)|}{\|x-y\|_{\Mm}} \leq \frac{|f(x)-f(y)|}{\|x-y\|_{\R^D}} \leq 1
\end{align*} as $\|x-y\|_{\R^D} \leq \|x-y\|_{\Mm}$ under an isometric embedding and hence $f \in \Lip_1(\Mm)$. Thus
 $$
    \E \left[W_1(Q, Q_n)\right] \leq \E \left[W_1^{\Mm}(Q, Q_n)\right] \leq C_\delta n^{-\frac{1}{d + \delta}}.
$$
\end{proof}

Finally, we prove our statistical estimation result in Theorem \ref{thm:stat}.

\begin{proof}[Proof of Theorem \ref{thm:stat}]
    Choose $\delta > 0$. Recall from Lemma \ref{lemma:biasvariance} we have
\begin{align*}
    W_1((\hat{g}_n)_{\sharp}\rho, Q) \leq \E \inf_{g_\theta \in \Gg} W_1((g_\theta)_{\sharp}\rho, Q) + 2\E W_1(Q_n, Q)
\end{align*} The first term is the approximation error which can be controlled within an arbitrarily small accuracy $\epsilon$. Theorem \ref{thm:approx} shows the existence of a neural network function $g_\theta$ with $O\left(\log\left(\frac{1}{\epsilon}\right)\right)$ layers and $O(D\epsilon^{-d/\alpha}\log(\frac{1}{\epsilon}))$ neurons such that $W_1((g_\theta)_\sharp \rho, Q) \leq \epsilon $ for any $\epsilon > 0$. We choose $\epsilon = n^{-\frac{1}{d+\delta}}$ to optimally balance the approximation error and the statistical error. The second term is the statistical error for which we recall from Lemma \ref{lemma:statconverge} that $\E \left[ W_1(\hat{Q}_n,Q) \right] \leq C_\delta n^{-\frac{1}{d+\delta}}$ for some constant $C_\delta$. 

Thus  we have
\begin{align*}
    \E W_1((\hat{g}_n)_\sharp \rho, Q) \leq n^{-\frac{1}{d+\delta}} + 2C_\delta  n^{-\frac{1}{d+\delta}} =  C n^{-\frac{1}{d+\delta}}
\end{align*} 
by setting $C = 1 + 2C_\delta$. This concludes the proof.
\end{proof}

\subsection{Controlling the noisy samples}
\label{SubsecNoisy}
In the noisy setting, we are given $n$ noisy i.i.d. samples $\hat{X}_1,...,\hat{X}_n$ of the form $\hat{X}_i = X_i + \xi_i$, for $X_i \sim Q$ and $\xi_i$ distributed according to some noise distribution. The optimization in \eqref{emprisk} is performed with the noisy empirical distribution $\hat{Q}_n = \frac{1}{n}\sum_{i=1}^n \delta_{\hat{X}_i}$. 
\begin{lemma}
\label{lemma:noisybiasvariance}
Under the same assumptions of Theorem \ref{thm:stat} and in the noisy setting,  we have
\begin{align}
\label{eq:biasvariancenoise}
    \E W_1((\hat{g}_n)_{\sharp}\rho, Q) \leq \inf_{g_{\theta} \in \Gg }W_1((g_{\theta})_{\sharp}\rho, Q) + 2 \E W_1(Q_n,Q) + 2\E W_1(\hat{Q}_n, Q_n)
\end{align} where $\hat{Q}_n$ is the noisy empirical distribution and $Q_n$ is the clean empirical distribution.
\end{lemma}

\begin{proof} Recalling the definition of $\hat{g}_n$ as the empirical risk minimizer, we compute
    \begin{align*}
\E W_1((\hat{g}_n)_{\sharp}\rho, Q)
    &\leq \E W_1((\hat{g}_n)_{\sharp}\rho, \hat{Q}_n) + \E W_1(\hat{Q}_n, Q) \\
    &\leq \E \inf_{g_\theta \in \Gg} W_1((g_\theta)_{\sharp}\rho, \hat{Q}_n) + \E W_1(Q_n, Q) + \E W_1(\hat{Q}_n, Q_n)\\
    &\leq \E \inf_{g_\theta \in \Gg} W_1((g_\theta)_{\sharp}\rho, Q) + 2\E W_1(Q_n, Q) + 2\E W_1(\hat{Q}_n, Q) \\
\end{align*}
since $W_1((\hat{g}_n)_\sharp \rho, \hat{Q}_n) = \inf_{g_\theta \in \mathcal{G}_{\rm NN}} W_1((g_\theta)_\sharp \rho, \hat{Q}_n)$ from \eqref{emprisk}.

\end{proof}

\begin{lemma}\label{lemma:noiseconverge}
    Write $ W_1(\hat{Q}_n, Q_n) = W_1^{\R^D}(\hat{Q}_n, Q_n) $. In the noisy setting, we express $\hat{X}_i = X_i + \xi_i$ where $X_i$ is drawn from $Q$ and then noised with $\xi_i$ drawn from some noise distribution. Then 
    \begin{align*}
        \mathbb{E}[W_1(Q_n, \hat{Q}_n)] \leq \sqrt{V_\xi}
    \end{align*} where $V_\xi = \E\|\xi\|_2^2$ which is the variance of the noise.
\end{lemma}

\begin{proof}
    Let $\hat{X}_i, X_i$ be samples defining $\hat{Q}_n, Q_n$  respectively. We have $\hat{X}_i = X_i + \xi_i$ where $\xi$ is the noise term. Compute
    \begin{align*}
       \E W_1(\hat{Q}_n, Q_n) &= \E\sup_{f \in \Lip_1(\R^D)}\hat{Q}_n(f)-Q_n(f) = \E\sup_{f \in \Lip_1(\R^D)}\frac{1}{n}\sum_{i=1}^n f(\hat{X}_i) - f(X_i) \\
        &\leq \E\sup_{f \in \Lip_1(\R^D)}\frac{1}{n}\sum_{i=1}^n |f(\hat{X}_i) - f(X_i)| = \E\sup_{f \in \Lip_1(\R^D)}\frac{1}{n}\sum_{i=1}^n |f(X_i + \xi_i) - f(X_i)| \\
        &\leq  \E\sup_{f \in \Lip_1(\R^D)}\frac{1}{n}\sum_{i=1}^n \|\xi_i\|_2 = \E \|\xi\|_2 \leq \sqrt{V_\xi}
    \end{align*} the last line follows from Jensen's inequality.
\end{proof} 

We conclude in the noisy setting that
\begin{align*}
    \E W_1((\hat{g}_n)_\sharp \rho, Q) \leq \epsilon_{\rm appx} +2 C_\delta n^{-\frac{1}{d+\delta}} + 2\sqrt{V} \leq Cn^{-\frac{1}{d+\delta}} + 2 \sqrt{V_\xi}
\end{align*} after balancing the approximation error $\epsilon_{\rm appx}$ appropriately.

\section{Conclusion}

We have established approximation and statistical estimation theories of deep generative models for estimating distributions on a low-dimensional manifold.
The statistical convergence rate in this paper depends on the intrinsic dimension of data. 
In light of the manifold hypothesis, which suggests many natural datasets lie on low dimensional manifolds, our theory rigorously explains why deep generative models defy existing theoretical sample complexity estimates and the curse of dimensionality. In fact, deep generative models are able to learn low-dimensional geometric structures of data, and allow for highly efficient sample complexity independent of the ambient dimension. Meanwhile the size of the required network scales %weakly(linearly) with the ambient dimension and at the expected exponential rate
exponentially with the intrinsic dimension.

Our theory imposes very little assumption on the target density $Q$, requiring only that it admit a density $q$ with respect to the volume measure and that $q$ is upper and lower bounded. In particular we make no smoothness assumptions on $q$. This is practical, as we do not expect existing natural datasets to exhibit high degrees of smoothness.

In this work we assume access to computation of the $W_1$ distance. However during GAN training a discriminator is trained for this purpose. It would be of interest for future work to investigate the low-dimensional role of such discriminator networks which approximate the $W_1$ distance in practice. 

Additionally, we provide an alternative approach to construct the oracle transport by decomposing the manifold into Voronoi cells and transporting the easy-to-sample distribution onto each disjoint cell directly in Appendix \ref{appsecvoronoi}.

\bibliographystyle{plainnat}

\bibliography{refs}

\begin{appendix}

\section{Deep ReLU Approximation of H\"older functions}
\label{sec:appendixapprox}
In this section, $\log$ denotes the base 2 logarithm by default. $\bigtimes_{i=1}^d$ denotes the Cartesian product of $d$ sets. The goal is to determine the approximation rate of deep ReLU networks for H\"older continuous functions. Let $f \in C^\alpha([0,1]^d)$ with H\"older norm $\| f \|_{C^\alpha}$ where $\alpha \in (0, 1)$. We first approximate $f$ by a piecewise constant function $f^n$ in Section \ref{sec:appendixpiece}, and then approximate $f^n$ by a deep ReLU Network $\Phi$ in Section \ref{sec:appendixtogether}.

\subsection{Piecewise constant approximation}
\label{sec:appendixpiece}
Let $[n]^d = \{ (a_1, a_2, \dots, a_d) :  a_i \in \mathbb{N}, 1 \leq a_i \leq n \}$. Given any $n \in \mathbb{N}$, we cover $[0,1]^d$ by $n^d$ non-overlapping open cubes. For any $(k_1, \dots, k_d) = \vec{k} \in [n]^d$, we define 
\begin{equation}
\label{eq:Qk}
Q_{\vec{k}} = \bigtimes_{i=1}^d \left(\frac{k_i - 1}{n}, \frac{k_i}{n}\right).
\end{equation}
\begin{lemma}
\label{lemma:pwconstant}
Let $f \in C^\alpha([0,1]^d)$ with H\"older norm $\| f \|_{C^\alpha}$. For any $n \in \mathbb{N}$, define 
\[f^n(x) = \sum_{k \in [n]^d} \left( n^d \int_{Q_{\vec{k}}} f(y) \, dy \right) \indicator_{Q_{\vec{k}}}(x).  \]
Then 
\[\|f - f^n \|_{L^1} < \| f \|_{C^\alpha} \frac{d^{\alpha / 2}}{n^\alpha}.\]
\end{lemma}
\begin{proof}
We estimate
\begin{align*}
\| f - f^n \|_{L^1}
    &= \int \left| f(x) - f^n(x) \right| \, dx \\
    &= \int \left|\sum_{k \in [n]^d} \left( n^d \int_{Q_{\vec{k}}} f(x) \, dy \right)\indicator_{Q_{\vec{k}}}(x) - \sum_{k \in [n]^d} \left( n^d \int_{Q_{\vec{k}}} f(y) \, dy \right) \indicator_{Q_{\vec{k}}}(x) \right| \, dx \\
    &= n^d \int \left|\sum_{k \in [n]^d} \left(\int_{Q_{\vec{k}}} (f(x) - f(y)) \, dy \right)\indicator_{Q_{\vec{k}}}(x) \right| \, dx \\
    &\leq n^d \sum_{k \in [n]^d} \int \indicator_{Q_{\vec{k}}}(x) \int_{Q_{\vec{k}}} |f(x) - f(y)| \, dy \, dx \\
    &= n^d \sum_{k \in [n]^d} \int_{Q_{\vec{k}}} \int_{Q_{\vec{k}}} |f(x) - f(y)| \, dy \, dx \\
    &\leq n^d \sum_{k \in [n]^d} \int_{Q_{\vec{k}}} \int_{Q_{\vec{k}}} \|f\|_{C^\alpha} \frac{d^{\alpha/2}}{n^\alpha} \, dy \, dx \\
    &= \|f\|_{C^\alpha} \frac{d^{\alpha/2}}{n^\alpha} n^d \sum_{k \in [n]^d} \frac{1}{n^{2d}} = \|f\|_{C^\alpha} \frac{d^{\alpha/2}}{n^\alpha}.
\end{align*}

where we use crucially use the fact that $\sup\limits_{x, y \in Q_{\vec{{k}}}} |f(x) - f(y)| \leq \|f\|_{C^\alpha}  \sup\limits_{x, y \in Q_{\vec{{k}}}} |x-y|^\alpha = \|f\|_{C^\alpha} \frac{d^{\alpha/2}}{n^\alpha}$.
\end{proof}

\subsection{Neural network approximation}
\label{sec:appendixnetwork}
We start with the well-known result originally stated in \cite{Yarotsky2017ErrorBF}.

\begin{lemma}
\label{lemma:times}
Let $A > 0$. For any $\epsilon \in (0, A^2)$, there is a ReLU network which implements a function $\tilde{\times}: \R^2 \rightarrow \R$ such that $$\sup\limits_{|x| \leq A, |y| \leq A} \left| \tilde{\times}(x, y) - xy \right| = \epsilon.$$ This network has depth at most $c \log\left(\frac{A^2}{\epsilon} \right)$, width at most $8$, and weights bounded by $A$ (where $c$ is an absolute constant).
\end{lemma}
\begin{proof}
The result follows from a careful reading of the proof in Appendix A.2 in \cite{Chen2019NonparametricRO}.
\end{proof}

The network given by Lemma \ref{lemma:times} approximates the multiplication of two numbers. We seek an approximation of the multiplication of $d$ numbers, and this is achieved by composing $\tilde{\times}$ with itself.

\begin{lemma}
\label{lemma:dtimes}
Fix $d \in \mathbb{N}$ and let $M > 0$. For any $\epsilon \in (0, M^2)$, there is a ReLU network which implements a function $\tilde{\times}_d : \R^d \rightarrow \R$ such that $$\sup\limits_{|x_1|, \dots, |x_d| \leq M} \left| \tilde{\times}(x_1, \dots, x_d) - x_1  \cdots x_d \right| < \epsilon.$$ This network has depth at most $c_1  \log\left(\frac{d^3 M^d}{\epsilon}\right) + c_2$, width at most $8d$, and weights bounded by $2M$ (where $c_1$ and $c_2$ are absolute constants).
\end{lemma}
\begin{proof}
% Let $\delta < \frac{\delta}{d^2 A^{d-2}}$.
Our idea is to realize the multiplication in a binary tree structure, illustrated in Figure \ref{fig:timesnet}.
We first assume that $2^{k-1} < d \leq 2^k$ for some integer $k$, and let $\delta = \frac{\epsilon}{4^{k-1} M^{2^k - 2}}$. We first handle the case that $d = 2^k$. We will construct a family of $k$ functions $\left\{ \tilde{\times}_{2^i} : \R^{2^i} \rightarrow \R \right\}_{i=1}^k$ iteratively. We will show that for all $1 \leq i \leq k$, the function $\tilde{\times}_{2^i}$ implements $2^i$-ary multiplication with error at most $4^{i-1}M^{2^i - 2} \delta$ (when $|x_i| < M$), at most $c \log\left( \frac{M^{2^{i+1} - 2}}{\delta^i} \right)$ layers, width at most $4\cdot2^i$, and weights bounded by $M$.

For $i = 1$, we define $\tilde{\times}_2$ to be the function defined in Lemma \ref{lemma:times} with the parameters $\epsilon = \delta$ and $A = M$. Then $\tilde{\times}_2$ has maximum error $\delta = 4^{1-1} M^{2^1 - 2} \delta$ and is implementable by a ReLU network with at most $c \log \left(\frac{A^2}{\epsilon} \right) = c \log \left(\frac{M^2}{\delta} \right) = c \log \left(\frac{M^{2^{1+1} - 2}}{\delta^1}\right)$ layers, width $8 = 4 \cdot 2^1$, and weights bounded by $M$, all as desired.

Now suppose the claim has been proven for $\tilde{\times}_{2^i}$. Let $\tilde{\times}$ be the function defined in Lemma \ref{lemma:times} with the parameters $\epsilon = \delta$ and $A = M^{2^i}$. Then $\tilde{\times}$ has $c \log \left(\frac{A^2}{\epsilon} \right) = c \log \left( \frac{M^{2^{i+1}}}{\delta} \right)$ layers, width 8, and weights bounded by $M$. We define 
\[\tilde{\times}_{2^{i+1}}(x_1, \dots, x_{2^{i+1}} )= \tilde{\times} \left(\tilde{\times}_{2^i}(x_1, \dots, x_{2^i}), \tilde{\times}_{2^i}(x_{2^i + 1}, \dots, x_{2^{i+1}}) \right). \] 

\begin{figure}[t!]
    \centering
    \begin{minipage}{0.45\textwidth}
        \centering
        \includegraphics[width=1.75in]{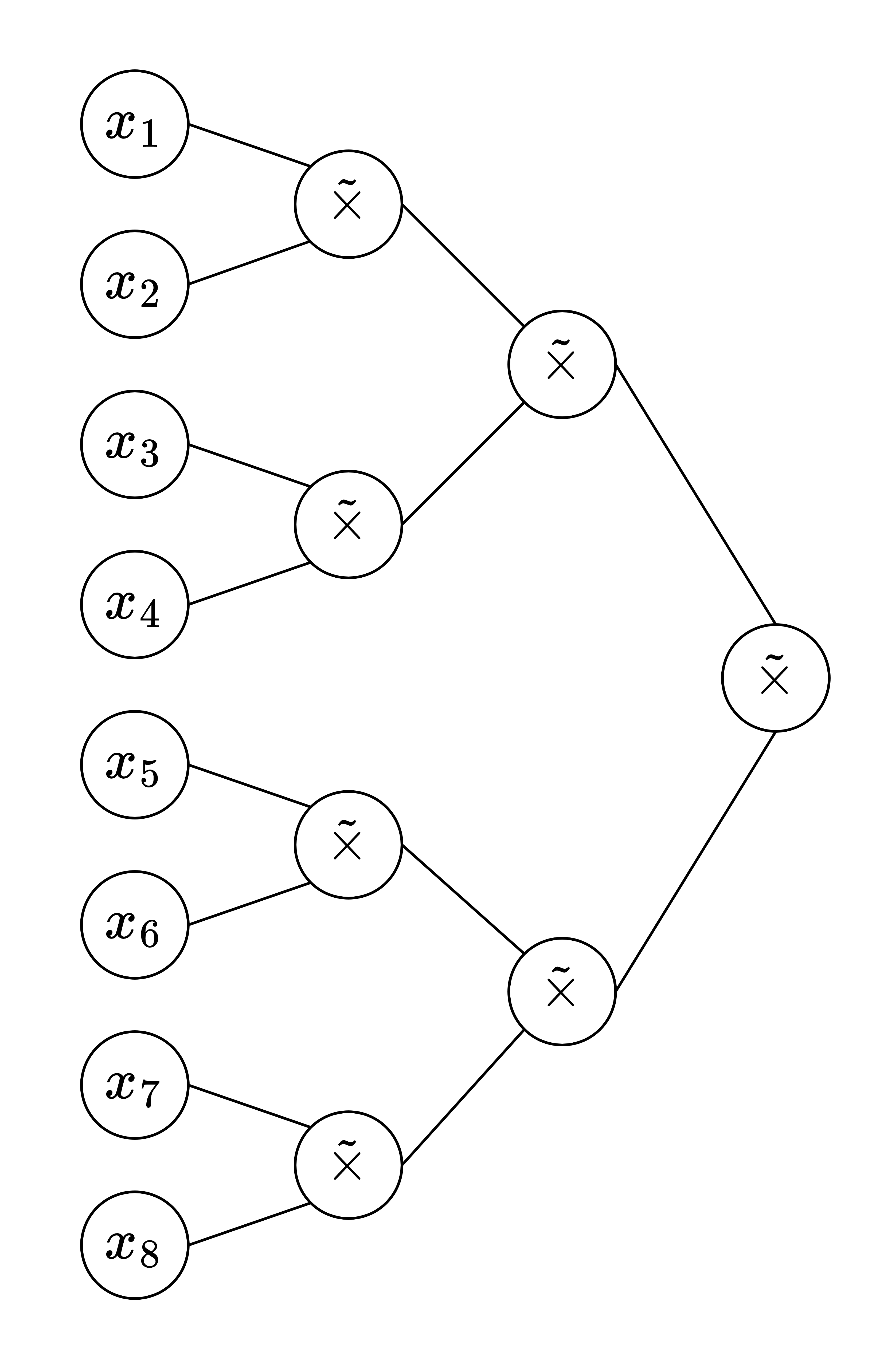}
        \caption*{(a) $\tilde{\times}_8$}
    \end{minipage}
    \begin{minipage}{0.45\textwidth}
        \centering
        \includegraphics[width=1.75in]{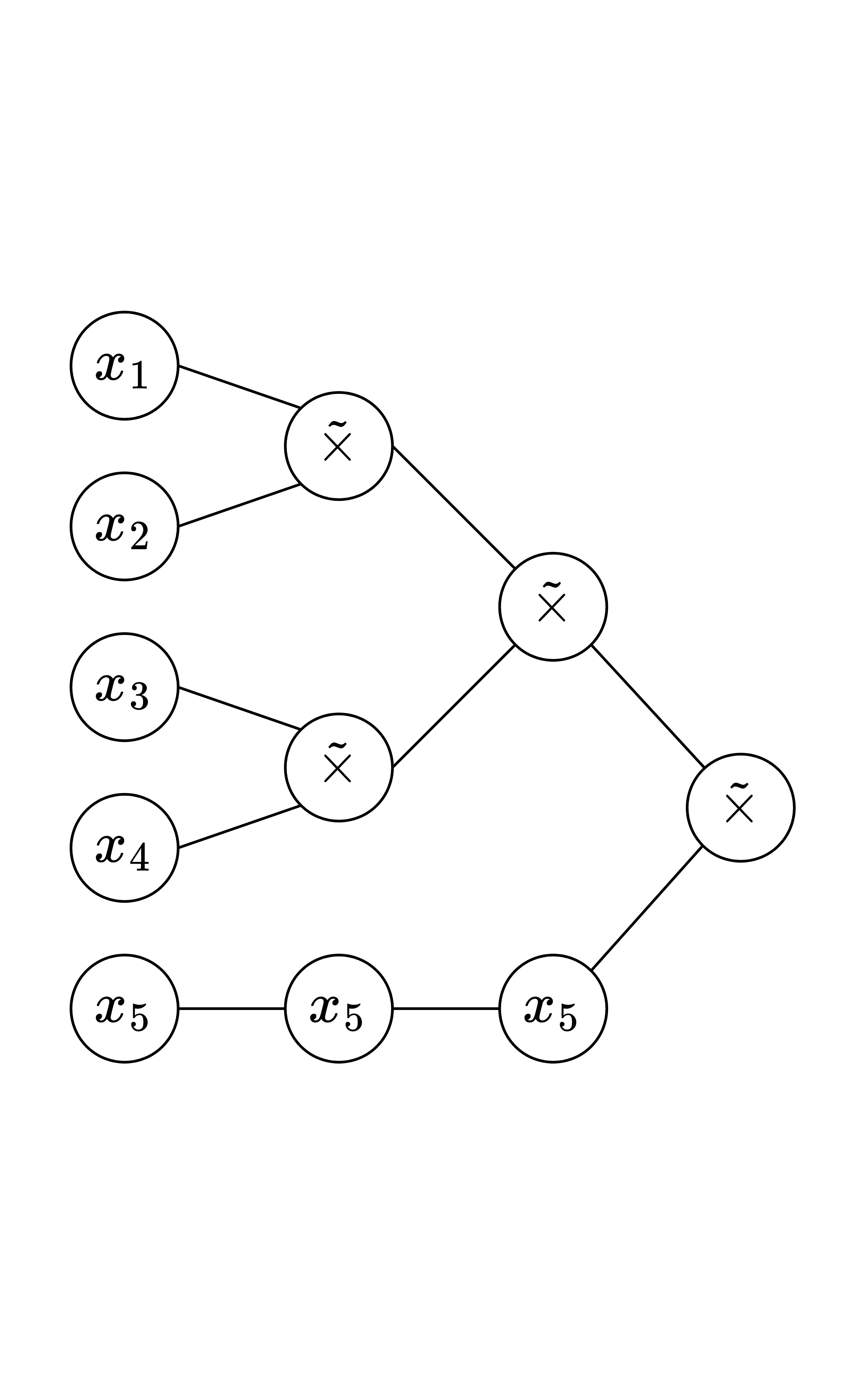}
        \caption*{(b) $\tilde{\times}_5$}
    \end{minipage}
    \caption{Network diagrams for ReLU networks approximating multiplication in Lemma \ref{lemma:dtimes}.}
    \label{fig:timesnet}
\end{figure}

Then $\tilde{\times}_{2^{i+1}}$ has depth $c \log\left( \frac{M^{2^{i+1} - 2}}{\delta^i} \right) + c \log \left( \frac{M^{2^{i+1}}}{\delta} \right) = c \log \left( \frac{M^{2^{i+2} - 2}}{\delta^{i+1}} \right)$, width $4 \cdot 2^i + 4 \cdot 2^i = 4 \cdot 2^{i+1}$, and weights bounded by $M$. It remains to compute the following error bound:

\begin{align*}
&|\tilde{\times}\left(\tilde{\times}_{2^i}(x_1, \dots, x_{2^i}), \tilde{\times}_{2^i}(x_{2^i+1}, \dots, x_{2^{i+1}} ) \right) - x_1 \cdots x_{2^{i+1}}| \\
    &\leq |\tilde{\times}\left(\tilde{\times}_{2^i}(x_1, \dots, x_{2^i}), \tilde{\times}_{2^i}(x_{2^i+1}, \dots, x_{2^{i+1}} ) \right) - \tilde{\times}_{2^i}(x_1, \dots, x_{2^i}) \cdot \tilde{\times}_{2^i}(x_{2^i+1}, \dots, x_{2^{i+1}} )| \\
    & \quad + |\tilde{\times}_{2^i}(x_1, \dots, x_{2^i}) \cdot \tilde{\times}_{2^i}(x_{2^i+1}, \dots, x_{2^{i+1}}) - \tilde{\times}_{2^i}(x_1, \dots, x_{2^i}) \cdot x_{2^i+1} \cdots x_{2^{i+1}}| \\
    & \quad + |\tilde{\times}_{2^i}(x_1, \dots, x_{2^i}) \cdot x_{2^i+1} \cdots x_{2^{i+1}} - x_1 \cdots x_{2^{i+1}}| \\
    & \leq \delta + |\tilde{\times}_{2^i}(x_1, \dots, x_{2^i})| \cdot |\tilde{\times}_{2^i}(x_{2^i+1}, \dots, x_{2^{i+1}}) - x_{2^i+1} \cdots x_{2^{i+1}}| \\
    &\quad  + |x_{2^i + 1} \dots x_{2^{i+1}}| \cdot |\tilde{\times}_{2^i}(x_1, \dots, x_{2^i}) - x_1 \cdots x_{2^i}| \\
    &\leq \delta + M^{2^i} (4^{i-1}M^{2^i - 2}\delta) + M^{2^i} (4^{i-1}M^{2^i - 2}\delta)\delta \\
    & = (1 + 2 4^{i-1} M^{2^{i+1} - 2} ) \delta \\
    & <  4^i M^{2^{i+1} - 2} \delta.
\end{align*}

From this, we have constructed a function $\tilde{\times}_{2^k}$ that approximates multiplication (of values $<M$) with error at most $4^{k-1} M^{2^k - 2} \delta = \epsilon$ that has depth 

\begin{align*}
c \log \left( \frac{M^{2^{k+1} - 2}}{\delta^k} \right) 
    &= c \log \left( \frac{M^{2^{k} - 2} M^{2^k} (4^{k-1})^k (M^{2^k - 2})^k }{\epsilon^k} \right) \\
    &= c \log \left( M^{2^k - 2} \right) +  c \log \left( \frac{(4^{k-1})^k (M^{2^k})^k }{\epsilon^k} \right) \\
    &= c \log \left( M^{2^k - 2} \right) +  c k \log \left( \frac{4^{k-1} M^{2^k} }{\epsilon} \right) \\
    &< c (1 + k) \log \left( \frac{4^{k-1} M^{2^k} }{\epsilon} \right) \\
    &< c_1 k \log \left( \frac{4^{k-1} M^{2^k} }{\epsilon} \right),
\end{align*}

For some absolute constant $c_1$. Now since $k = \log(d)$, we have that $M^{2^k} = M^d$ and $4^{k-1} < 4^k = d^2$, so the ReLU network has depth at most $c \log(d) \log \left( \frac{d^2 M^d}{\epsilon} \right)$ where $c$ is an absolute constant (the same constant as in Lemma \ref{lemma:times}). The width of $\tilde{\times}_{2^k}$ is $4 \cdot 2^k = 4d$, and the weights are bounded by $M$.

Figure \ref{fig:timesnet}(a) shows a neural network diagram for the ReLU network implementing $\tilde{\times}_8$, which has the structure of a full binary tree. In order to handle numbers that are not powers of two, we use an architecture similar to the diagram in Figure \ref{fig:timesnet} (b) which depicts the ReLU network implementing $\tilde{\times}_5$.

Formally, suppose we have $2^{i-1} < d \leq 2^i$ for some $i \in \mathbb{N}$. Then consider the network $\tilde{\times}_{2^{i}}$ defined as before, but we remove the last $2^i - d$ input neurons, and replace them with $1$ everywhere they appear. Note that this can be achieved by adjusting the bias of each neuron appropriately. For example, any neuron can be turned into a constant 1 by making the weight vector $0$ and making the bias equal to $1$. This procedure will not affect the number of layers, it will not increase the width, and the parameters are bounded by $M + 1 < 2M$. Noting that $2^i < 2d$, we see that the ReLU network has width at most $4 \cdot 2^i < 4 \cdot 2d = 8d$. Finally, the depth is at most (for $c_1$ and $c_2$ absolute constants)

\begin{align*}
c \log(2^i) \log \left( \frac{(2^i)^2 M^d}{\epsilon} \right) 
    &<  c \log(2d) \log \left( \frac{(2d)^2 M^d}{\epsilon} \right) \\
    &= c(\log(2) + \log(d))\left(\log(4) + \log\left(\frac{d^2M^d}{\epsilon}\right)\right) \\
    &< c\left(\log(8) + 3\log(d)\log\left(\frac{d^2M^d}{\epsilon}\right)\right) \\
    &= c_1 \log\left( \frac{d^3 M^d}{\epsilon} \right) + c_2.
\end{align*}
\end{proof}

Next we approximate the indicator functions of intervals (which we denote by $\indicator_{(a,b)}$).

\begin{lemma}
\label{lemma:indicator}
Fix $M > 1$. Let $[a, b] \subseteq [-M, M]$ and $\epsilon < \frac{1}{2}(b-a)$. Then there is a ReLU network which implements a function $\tilde{\indicator}_{(a,b)}^\epsilon$ such that $$\left\|\tilde{\indicator}_{(a,b)}^\epsilon - \indicator_{(a,b)}\right\|_{L^1} = \epsilon.$$  This network has depth at most $c \log\left( \frac{1}{\epsilon} \right)$, width equal to 4, and weights bounded by $M$ (where $c$ is a constant depending only on $M$).
\end{lemma}
\begin{proof}
We define the ReLU network function
\[ \tilde{\indicator}^{\epsilon}_{(a, b)} (x) = \frac{1}{\epsilon} (\sigma(x - (a - \epsilon)) -  \sigma(x - a) -  \sigma(x - b) +  \sigma(x - (b + \epsilon))),\] 
where $\sigma$ is the ReLU activation function. Figure \ref{fig:indicatorapprox} is a plot of $\tilde{\indicator}^{\epsilon}_{(a, b)}$. Then it is clear that $$\left\| \tilde{\indicator}^{\epsilon}_{(a, b)} -  \indicator_{(a, b)}\right\|_{L^1} = \epsilon.$$ 

\begin{figure}[t]
    \centering
    \includegraphics[height=1.4in]{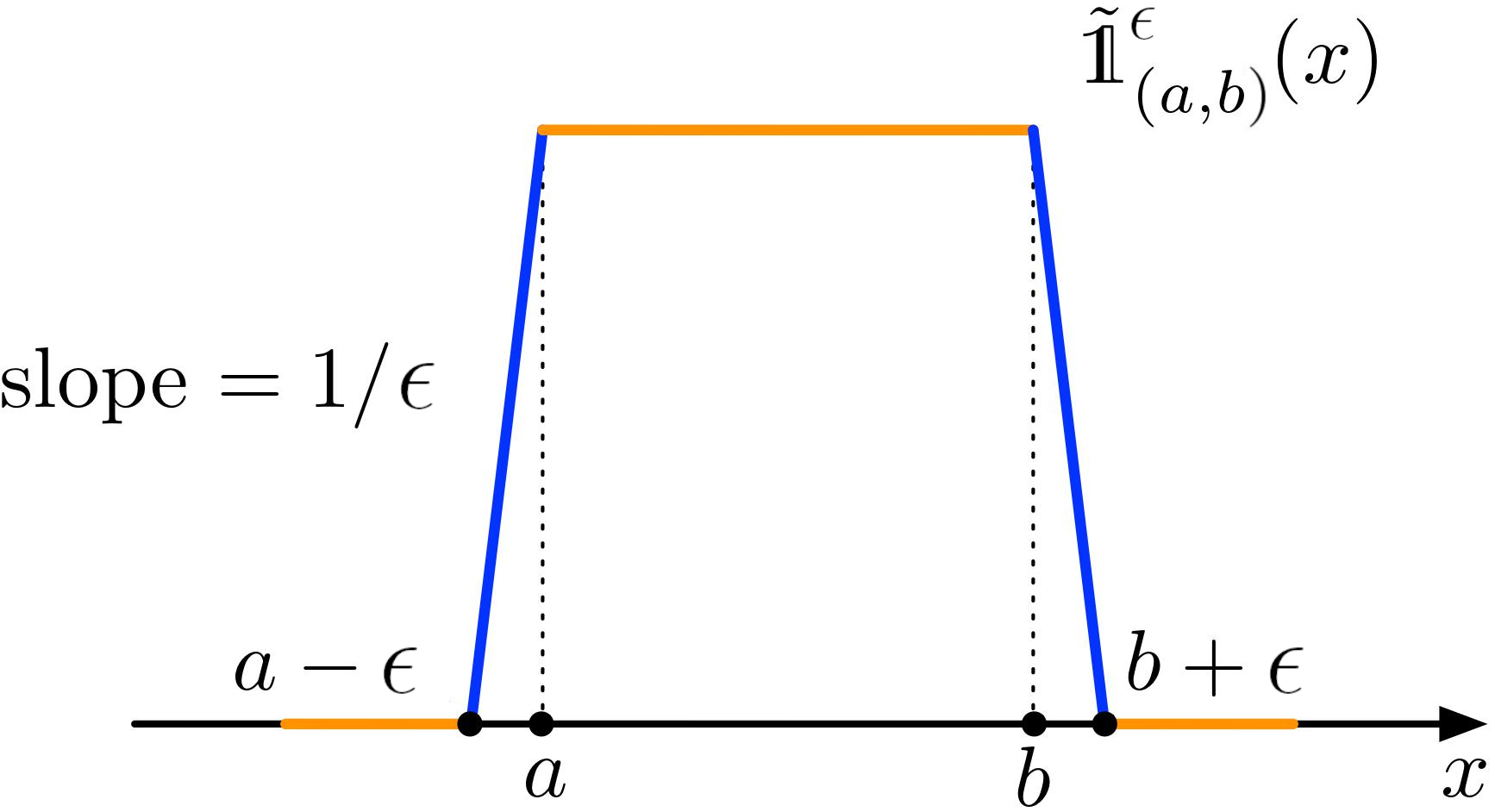}
    \caption{Plot of $\tilde{\indicator}^{\epsilon}_{(a, b)}$.}
    \label{fig:indicatorapprox}
\end{figure}

Note that we can express $\frac{1}{\epsilon}$ as

\[ \frac{1}{\epsilon} = M^{\left\lceil \log_M\left(\frac{1}{\epsilon}\right) \right\rceil - 1} \cdot \frac{\frac{1}{\epsilon}}{M^{\left\lceil \log_M\left(\frac{1}{\epsilon}\right) \right\rceil - 1}} = \underbrace{M \times M \times \cdots \times M}_{\left\lceil \log_M\left(\frac{1}{\epsilon}\right) \right\rceil - 1 \text{ times}} \times R, \]

where $0 < R \leq M$. This is a product of $\left\lceil \log_M\left(\frac{1}{\epsilon}\right) \right\rceil$ numbers that are all bounded by $M$. Then note that

\[ \tilde{\indicator}^{\epsilon}_{(a, b)} = M \times \dots \times M \times R (\sigma(x - (a - \epsilon)) -  \sigma(x - a) -  \sigma(x - b) +  \sigma(x - (b + \epsilon))) \]

can be implemented by ReLU network with $1 + \left\lceil \log_M\left(\frac{1}{\epsilon}\right) \right\rceil$ layers. The first layer has $4$ neurons, and the second layer has one neuron, and they together compute $R (\sigma(x - (a - \epsilon)) -  \sigma(x - a) -  \sigma(x - b) +  \sigma(x - (b + \epsilon)))$. Then the next $\left\lceil \log_M\left(\frac{1}{\epsilon}\right) \right\rceil - 1$ each multiply this value by $M$ (since all values at this point are positive, the ReLU activation does nothing at each layer). Thus the ReLU network has width $4$ (though only the first layer has more than one neuron) and weights bounded by $M$.
\end{proof}

We combine Lemma \ref{lemma:dtimes} and Lemma \ref{lemma:indicator} to obtain an approximation to the indicator function of $d$-dimensional cube.
\begin{lemma}
\label{lemma:indicatord}
Fix $M > 1$. Let $Q = \bigtimes_{k=1}^d (a_k, b_k) \subseteq [-M, M]^d$ be a bounded $d$-dimensional cube (i.e. $b_1 - a_1 = b_k - a_k$ for all $k \in [d]$), and suppose $\epsilon < \min \left(\frac{b_1 - a_1}{2}, 1\right)$. Then there exists a function $\phi : [-M, M]^d \rightarrow \R$ implementable by a ReLU network such that 
\[ \int_{[-M, M]^d} |\indicator_Q (x) - \phi (x) | \, dx < \epsilon. \]

The network has depth at most $c_1 \log\left( \frac{d^2 4^d}{\epsilon} \right) + c_2$, width at most $4d$, and weights bounded by $\max\{M, 2\}$ (where $c_1$ and $c_2$ are constants only depending on $M$).
\end{lemma}
\begin{proof}
Denote by $\tilde{\indicator}^\delta_{(a_i, b_i)}$ the approximation to $\indicator_{(a_i, b_i)}$ obtained from Lemma \ref{lemma:indicator} with $\delta = \frac{\epsilon}{2}$. Let $\eta = \frac{\epsilon}{2^{d+1} M^d}$, and denote by $\tilde{\times}_d$ the approximation of the multiplication of $d$ factors obtained from Lemma \ref{lemma:dtimes} with parameters $M = 1$ and error $\eta$ (which is denoted as $\epsilon$ in the lemma statement). Then we define $\phi$ by
\[ \phi(x_1, \dots, x_d) = \tilde{\times}_d\left(\tilde{\indicator}^\delta_{(a_1, b_1)}(x_1), \dots, \tilde{\indicator}^\delta_{(a_d, b_d)}(x_d)\right). \]

We compute
\begin{align*}
& \int_{[-M, M]^d} \left| \indicator_{\bigtimes_{i=1}^d [a_i, b_i)}(x) - \phi(x) \right| \, dx \\
    &= \int_{[-M, M]^d} \left| \prod_{i=1}^d \indicator_{(a_i, b_i)}(x_i) - \tilde{\times}_d\left(\tilde{\indicator}^\delta_{(a_1, b_1)}(x_1), \dots, \tilde{\indicator}^\delta_{(a_d, b_d)}(x_d)\right) \right| \, dx\\
    &\leq \int_{[-M, M]^d} \left| \prod_{i=1}^d \indicator_{(a_i, b_i)}(x_i) - \prod_{i=1}^d \tilde{\indicator}^\delta_{(a_i, b_i)}(x_i) \right| \, dx \\
    & \quad \quad + \int_{[-M, M]^d} \left| \prod_{i=1}^d \tilde{\indicator}^\delta_{(a_i, b_i)}(x_i)  - \tilde{\times}_d\left(\tilde{\indicator}^\delta_{(a_1, b_1)}(x_1), \dots, \tilde{\indicator}^\delta_{(a_d, b_d)}(x_d)\right) \right| \, dx \\
    &< \int _ {\bigtimes_{i=1}^d [a_i - \delta, b_i + \delta] \, \setminus \, \bigtimes_{i=1}^d [a_i, b_i)} \left|\prod_{i=1}^d \tilde{\indicator}^\delta_{(a_i, b_i)}(x_i)\right| + \int_{[-M, M]^d} \frac{\epsilon}{2^{d+1} M^d}  \, dx \\
    &< \text{Vol}\left (\bigtimes_{i=1}^d (a_i - \delta, b_i + \delta) \, \setminus \, \bigtimes_{i=1}^d (a_i, b_i) \right) + \frac{\epsilon}{2}
\end{align*}

where last inequality follows since $\left|\prod_{i=1}^d \tilde{\indicator}^\delta_{(a_i, b_i)}(x_i)\right| < 1$. We now focus on bounding the measure of $R = \left (\bigtimes_{i=1}^d (a_i - \delta, b_i + \delta) \, \setminus \, \bigtimes_{i=1}^d (a_i, b_i) \right)$. First we express $R = \bigcup_{k=1}^d S_k^+ \cup \bigcup_{k=1}^d S_k^-$ where

\[ S_k^- = \left(\bigtimes_{i=1}^{k-1} (a_i - \delta, b_i + \delta)\right) \times (a_k - \delta, a_k) \times \left(\bigtimes_{i=k+1}^{d} (a_i - \delta, b_i + \delta)\right), \]

\[ S_k^+ = \left(\bigtimes_{i=1}^{k-1} (a_i - \delta, b_i + \delta) \right) \times (b_k, b_k + \delta) \times \left(\bigtimes_{i=k+1}^{d} (a_i - \delta, b_i + \delta)\right). \]

Thus we have that  $\text{Vol}(R) \leq \sum_{k=1}^d \text{Vol}(S_k^+) + \text{Vol}(S_k^-) = 2(2\delta + b_i - a_i)^{d-1} \delta < 2(2\delta + 2M)^{d-1}\delta$. If we pick $\delta = \frac{\epsilon}{4(3M)^{d-1}}$, then we have 
\begin{align*}
\|\indicator_Q - \phi \|_{L^1}
    &< \text{Vol}\left (\bigtimes_{i=1}^d (a_i - \delta, b_i + \delta) \, \setminus \, \bigtimes_{i=1}^d (a_i, b_i) \right) + \frac{\epsilon}{2} \\
    &\leq 2(2\delta + 2M)^{d-1} \delta + \frac{\epsilon}{2} \\
    &= 2(3M)^{d-1} \frac{\epsilon}{4(3M)^{d-1}} + \frac{\epsilon}{2} \\
    &=\frac{\epsilon}{2}  + \frac{\epsilon}{2} = \epsilon.
\end{align*}

Finally, we determine the size of $\phi$. Each $\tilde{\indicator}^\delta_{(a_i, b_i)}$ can be implemented by a ReLU network with depth $c_1 \log\left(\frac{1}{\delta}\right) = c_1 \log\left(\frac{4 \cdot 3^{d-1}M^{d-1}}{\epsilon}\right) \leq c_2 \log\left(\frac{3^{d-1}M^{d-1}}{\epsilon}\right)$, width $4$, and weights bounded by $M$. $\tilde{\times}_d$ can be implemented by a ReLU network with depth at most
\[ c_3 \log\left(\frac{d^3 2^d}{\eta} \right) + c_4 = c_3 \log\left(\frac{d^3 2^d 2^{d+1} M^d}{\epsilon} \right) + c_4 = c_3 \log\left( \frac{d^3 2^{2d+1} M^d}{\epsilon} \right) + c_4, \]

width at most 4d, and weights bounded by 2. Thus $\phi$ has width at most $4d$, weights bounded by $\max\{M, 2\}$, and depth at most
\begin{align*}
& c_2 \log\left(\frac{3^{d-1}M^{d-1}}{\epsilon}\right) + c_3 \log\left( \frac{d^3 2^{2d+1} M^d}{\epsilon} \right) + c_4 \\
&< \max(c_2, c_3) \log\left(\frac{d^3 3^{d-1} 2^{2d+1} M^{2d-1}}{\epsilon^2}\right) + c_4 \\
&< \max(c_2, c_3) \log\left(\frac{d^3 4^{d-1} 4^{d+1} M^{2d}}{\epsilon^2}\right) + c_4 \\
&< \max(c_2, c_3) \log\left(\frac{(d^2 4^d M^d)^2}{\epsilon^2}\right) + c_4 \\
&= c_5 \log \left( \frac{d^24^dM^d}{\epsilon} \right) + c_4 \\
&= c_5 \log \left( \frac{d^24^d}{\epsilon} \right) + c_5 d \log(M) + c_4. \\
&< c_6 \log \left( \frac{d^24^d}{\epsilon} \right) + c_4.
\end{align*}
\end{proof}

Now we use the construction of Lemma \ref{lemma:indicatord} to approximate the function from Lemma \ref{lemma:pwconstant} as follows. Let $\{ Q_{\vec{k}} \}_{\vec{k} \in [n]^d}$ be a decomposition of $[0, 1]^d$ into almost non-overlapping cubes as in Lemma \ref{lemma:pwconstant}. The only overlap of these cubes is a set of measure $0$.

\begin{lemma}
\label{lemma:pwapprox}
Let $0 < \epsilon < \frac{1}{2n}$. Let $\{ \beta_{\vec{k}} \}_{\vec{k} \in [n]^d}$ be constants within $[-M, M]$. Then the function $g$ defined by 
\[ g(x) = \sum_{k \in [n]^d} \beta_{\vec{k}} \indicator_{Q_{\vec{k}}}(x) \]
can be approximated by a neural network $\tilde{g}$ with depth $c_1 \log\left( \frac{d^2 4^d n^d}{\epsilon}\right) + c_2$, width $4d n^d$, and weights bounded by $\max\{M, 2\}$ (where $c_1$ and $c_2$ are constants only depending on $M$), such that
\[ \int_{[-M, M]^d} |g(x) - \tilde{g}(x)| \, dx < \epsilon.  \]
\end{lemma}
\begin{proof}
For every $\vec{k} \in [n]^d$, let $\phi_{\vec{k}}$ be the function from Lemma \ref{lemma:indicatord} that approximates $\indicator_{Q_{\vec{k}}}$ with error $\frac{\epsilon}{M n^d}$. Then each $\phi_{\vec{k}}$ can be implemented by a ReLU network with 
\[ c_1 \log\left( \frac{d^2 4^d n^d M}{\epsilon}\right) + c_2 = c_1 \log\left( \frac{d^2 4^d n^d}{\epsilon}\right) + c_3 \] 
layers, width at most $4d$, and weights bounded by $2$. This means we can implement the function $\tilde{g}(x) = \sum_{k \in [n]^d} \beta_{\vec{k}} \phi_{\vec{k}}$ by a ReLU network with one more layer, width at most $4dn^d$, and weights bounded by $\max\{M, 2\}$. We compute

\begin{align*}
\int_{[-M, M]^d} |g(x) - \tilde{g}(x)| \, dx 
    &= \int_{[-M, M]^d} \left|\sum_{k\in[n]^d} \beta_{\vec{k}}\indicator_{Q_{\vec{k}}}(x) - \sum_{k\in[n]^d} \beta_{\vec{k}} \phi_{\vec{k}}(x)\right| \, dx \\
    &\leq \sum_{k\in[n]^d} |\beta_{\vec{k}}| \int_{[-M, M]^d} \left|\indicator_{Q_{\vec{k}}}(x) - \phi_{\vec{k}}(x)\right| \, dx \\
    &\leq \sum_{k\in[n]^d} |\beta_{\vec{k}}| \left( \frac{\epsilon}{Mn^d} \right) \\
    &\leq \sum_{k\in[n]^d} M \left( \frac{\epsilon}{Mn^d} \right) = \epsilon. \\
\end{align*}
\end{proof}

\subsection{Putting approximations together}
\label{sec:appendixtogether}
We combine Lemma \ref{lemma:pwconstant} with Lemma \ref{lemma:pwapprox} to obtain our approximation result. The requirement of $d \geq 4$ in the following lemma is for technical reasons, and is not a requirement for Lemma \ref{lemma:main} which is used in the main paper.

\begin{lemma}
\label{lemma:technical}
Suppose $f \in C^\alpha([0,1]^d)$, $\alpha \in (0, 1)$, with $\|f\|_{L^\infty} < M$ and $M \geq 2$. Assume further that $d \geq 4$. Let $\epsilon > 0$. Then there exists a function $\Phi$ implementable by a ReLU network such that $$\|f - \Phi\|_{L^1} < \epsilon.$$ The ReLU network has depth at most $c_1 d \log\left( \frac{8d \|f\|_{C^\alpha}^{1/\alpha}}{\epsilon^{2/\alpha}} \right) + c_2$, width at most $\frac{(4d)^{d/2} \|f\|_{C^\alpha}^{d/\alpha}}{\epsilon^{d/\alpha}}$, and weights bounded by $M$ (where $c_1$ and $c_2$ are constants only depending on $M$).

\end{lemma}
\begin{proof}
Let $n = \left\lceil \left(\frac{\| f \|_{C^{\alpha}}}{\epsilon}\right)^{1/\alpha} \sqrt{d} \right\rceil$. Let $f^n$ be the piecewise constant function from Lemma \ref{lemma:pwconstant}. Notice that $f^n$ follows the same form as the function $g$ in Lemma \ref{lemma:pwapprox}.   By Lemma \ref{lemma:pwapprox}, there is a ReLU network function $\Phi$ that approximates $f^n$  such that
\[ \int_{[-M, M]^d} | f^n(x) - \Phi(x) | \, dx < \frac{\epsilon}{2}. \]

Then we compute
\[ \|f - \Phi\|_{L^1} \leq \|f - f^n\|_{L^1} + \|f^n - \Phi\|_{L^1} < \frac{\|f\|_{C^{\alpha}} d^{\alpha/2}}{n^\alpha} + \frac{\epsilon}{2} < \frac{\epsilon}{2} + \frac{\epsilon}{2} = \epsilon. \]

Note that $\Phi$ has width $4dn^d < 4d \left( 2\left(\frac{\| f \|_{C^{\alpha}}}{\epsilon}\right)^{1/\alpha} \sqrt{d}\right)^d = c_0 \frac{(4d)^{d/2} \|f\|_{C^\alpha}^{d/\alpha}}{\epsilon^{d/\alpha}}$, and the weights of $\Phi$ are bounded by $M$. The depth of $\Phi$ is bounded by 
\begin{align*}
c_1 \log\left( \frac{d^2 4^dn^d}{\epsilon} \right) + c_2
    &< c_1 \log\left( \frac{d^2 4^d \|f\|_{C^\alpha}^\frac{d}{\alpha} 2^d d^{\frac{d}{2}}}{\epsilon \epsilon^{\frac{d}{\alpha}}} \right) + c_2 \\
    &< c_1 \log\left( \frac{d^d 8^d \|f\|_{C^\alpha}^\frac{d}{\alpha}}{\epsilon^{\frac{2d}{\alpha}}} \right) + c_2 \\
    &= c_1 d \log\left( \frac{8 d \|f\|_{C^\alpha}^{1/\alpha}}{\epsilon^{2 / \alpha}} \right) + c_2\\
\end{align*}
where we use the fact that $d \geq 4$ to bound $2 + \frac{d}{2} \leq d$.

\end{proof}

In the main paper, we use Lemma \ref{lemma:main}, which is proved below.

\begin{proof}[Proof of Lemma \ref{lemma:main}]
By allowing the constants to depend on all values except $\epsilon$, the depth of the ReLU network $\Phi$ from Lemma \ref{lemma:technical} can be expressed as
\[ c_1 d \log \left(\frac{8d \|f\|^{1/\alpha}_{C_\alpha}}{\epsilon^{2/\alpha}}\right) + c_2 = c_3 \left(\log \left(\frac{1}{\epsilon^{2/\alpha}}\right) + \log\left(8d \|f\|^{1/\alpha}_{C_\alpha}\right) \right) + c_2 = \frac{2}{\alpha} c_3 \log \left(\frac{1}{\epsilon}\right) + c_4 < c_5 \log \left(\frac{1}{\epsilon}\right), \]\
and the width can be expressed as

\[ \frac{(4d)^{d/2}\|f\|_{C^\alpha}^{d/\alpha}}{\epsilon^{d/\alpha}} = \frac{c_6}{\epsilon^{d/\alpha}}.\]
\end{proof}

\section{An Alternative Approach: Voronoi Decompositions}
\label{appsecvoronoi}

In this section we sketch an alternative approach to our approximation theory which decomposes the manifold into Voronoi cells and transports onto each disjoint cell directly. This theory is not complete as we have not ruled out the existence of particularly pathological cells which lack an open kernel in tangent space. However we believe such an approach is workable, though very technical, and has benefits associated with a disjoint Voronoi decomposition of the  manifold. Therefore we present a sketch of our ideas.

Recall our goal in approximation theory is to construct a transport map $g^*: (0,1)^{d+1} \to \Mm$ that can be approximated by neural networks. First, we partition the manifold into geodesic Voronoi cells and map local distributions over the cells onto tangent planes. Then we transform the distribution on tangent planes to a distribution on a ball from which we can apply optimal transport theory to produce a map between the source distribution on $(0,1)^{d}$ to a ball. Finally we glue together these local maps with indicators and a uniform random sample from $(0,1)$. We proceed with the first step of partitioning the manifold.

\textbf{Step 1: Voronoi decomposition.} Given $\Mm$ under Assumption \ref{assumptionmanifold}, we decompose $\Mm$ into a partition of finite geodesic Voronoi cells. This is done by first covering $\Mm$ with Euclidean balls around each point with uniform radius less than the global injectivity radius. Via compactness we can extract a finite subcover $\{U_j\}_{j=1}^J$ with centers $\{c_j\}_{j=1}^J$. We form a Voronoi partition of $\Mm$ denoted $\{V_j\}$ by taking the $c_j$ as centers, i.e. 
\begin{align*}
    V_j = \{x \in \Mm : d_\Mm(x, c_j) =\textstyle \min_{i \in [J]} d_\Mm(x, c_i) \}.
\end{align*} 
A Voronoi decomposition is illustrated in Figure \ref{fig:voronoi}. Note the boundary/overlap of the cells has measure $0$ with respect to $Q$. We check this later in the appendix. Now, given a distribution $Q$ on $\Mm$ with density $q$, we can define local distributions $Q_j$ as the conditional distribution of $Q$ on $V_j$.

For any measurable set $A$, we have
\begin{align*}
Q(A) = \sum_{j=1}^J Q(A \cap V_j) = \sum_{j=1}^j Q(V_j) Q(A | V_j) = \sum_{j=1}^j Q(V_j) Q_j(A).
\end{align*}

\begin{figure}
    \begin{minipage}[c]{0.43\textwidth}
    \centering
        \includegraphics[height=0.9in]{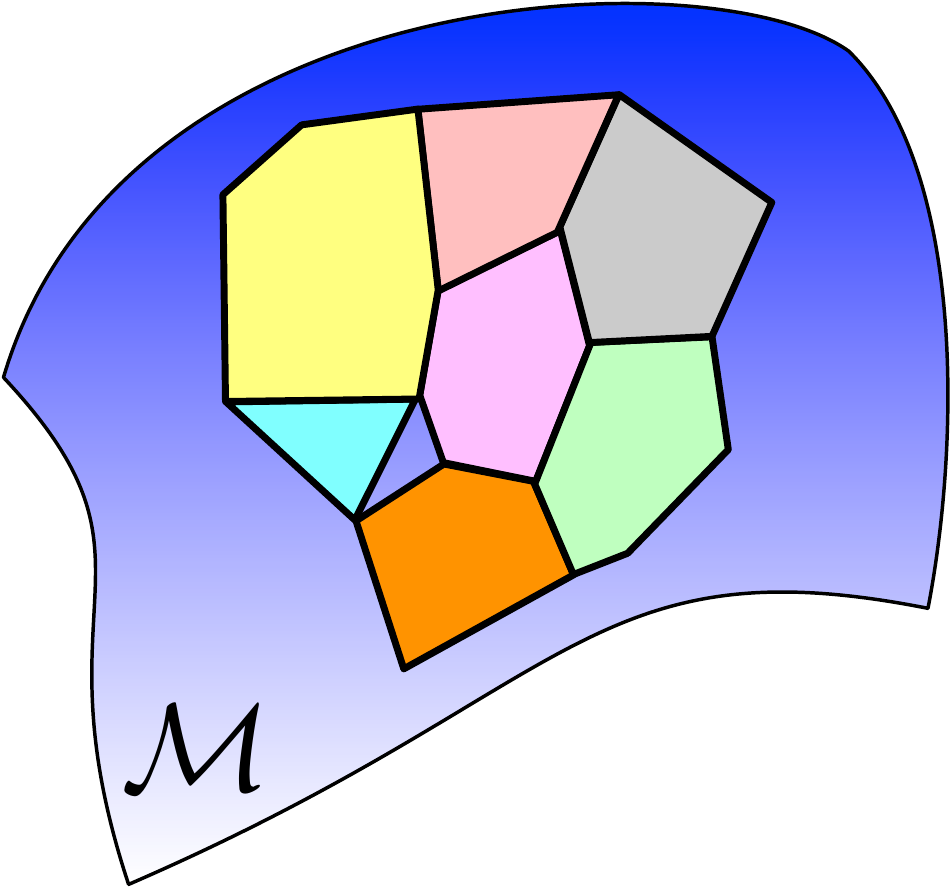}
        \caption{A Voronoi decomposition of $\Mm$.}
        \label{fig:voronoi}
    \end{minipage}
    ~
    % \hspace{0.1in}
    \begin{minipage}[c]{0.55\textwidth}
        \centering
        \includegraphics[height=0.9in]{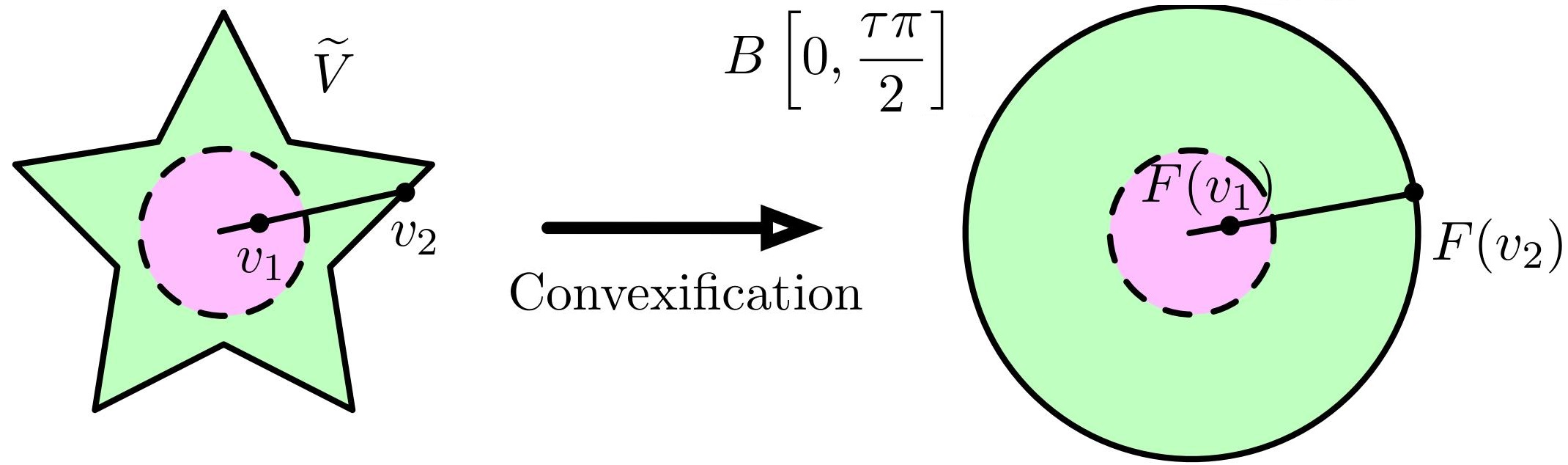}
        \caption{Expansion of star-shaped set into ball via $F$.}
        \label{fig:expansion}
    \end{minipage}
\end{figure}

\textbf{Step 2: Defining local lower-dimensional distributions.} Next we transform local distribution $Q_j$ on $V_j$ into a distribution $\tilde{Q}_j$ on $\tilde{V}_j = \exp_{c_j}^{-1}(V_j)$ which is inverse of image of the $j$th Voronoi cell under the exponential map. In particular $Q_j$ is the pushforward of $\tilde{Q}_j$ by the exponential map. The density of $\tilde{Q}_j$ given by
\begin{align*}
    \tilde{q}_j(v) = q_j(\exp_{c_j}(v))\left|\det J_{\exp_{c_j}}(v)\right|
\end{align*} via an application of formula \eqref{localdensity}. For simplicity, we use $\det J_{\exp_{c_j}}(v)$ to denote the Jacobian $\sqrt{\det g_{ij}^x(v)}$. $\tilde{Q}_j$ has two important properties. First, $\tilde{q}_j$ continues to be upper and lower-bounded and so we can still hope to apply optimal transport theory. Second, $\tilde{V}_j \subseteq T_{c_j}\Mm$ is a $d$-dimensional set on which $(\exp_{c_j})_\sharp\tilde{Q}_j = Q_j$. %Now we may hope to apply optimal transport theory and approximate the transport by neural networks without dependence on the high ambient dimension $D$. 
Ideally, to apply optimal transport theory, the target domain $\tilde{V}_j$ needs to be convex. However, $\tilde{V}_j$ is not convex here, but it is star-shaped. Recall a set $S$ is star-shaped if $\exists$ $x \in S$ such that for all $y \in S$, the set $[x:y] = \{\lambda x + (1-\lambda)y : \lambda \in [0,1]\} \subseteq S$. For an excellent reference on star-shaped sets see \citet{article}. This property is crucial and allows us to transform our target domain $\tilde{V}_j$ to a ball via a bi-Lipschitz transformation.

\textbf{Step 3: Convexifying the Target Domain $\tilde{V}_j$}. As we claim above, the target domain $\tilde{V}_j$ is star-shaped and centered at 0. $\tilde{V}_j$ is also bounded and closed and hence compact. Further we know for some $\eta > 0$, $B(0,\eta) \subseteq \tilde{V}_j$. Then we can apply the following lemma to produce a bi-Lipschitz transformation $F: \tilde{V}_j \to B\left(0,\frac{\tau \pi}{2}\right)$, where $\tau$ is the reach of $\Mm$.

\begin{lemma}\label{lemma:star}
    Let $S$ be a compact star-shaped set centered at $0$. Further suppose $\exists \eta > 0$ such that $B(0,\eta) \subseteq S$. Let $L > 0$ be such that $S \subseteq B[0,L]$. Then there exists $F:S \to B[0,L]$ such that $F$ is bi-Lipschitz with constant $C$ dependent on $L, \eta$.
\end{lemma}

Lemma \ref{lemma:star} is proved in Appendix \ref{proof:star}. The main idea is that the bi-Lipschitz map is constructed explicitly by expanding along rays (see Figure \ref{fig:expansion}). Theorem 1 from \citet{10.2307/2316022} tells us the radial function on $S$ which computes the radius of $\tilde{V}_j$ in a direction $u \in \mathbb{S}^{d-1}$ is Lipschitz when the star-shaped set is compact \textbf{and contains a nontrivial interior kernel}. We recall the kernel of a star-shaped is the set of points which have visibility to all points in the set. We believe appropriately chosen manifold Voronoi cells will have non-empty interior kernel. However this remains a technical difficulty.

From this we can construct bi-Lipschitz $F_j$ by fixing points in a ball of radius ${\eta}/{2}$ and expanding points at a rate dependent on its distance from $B(0,\eta/2)$ and $R(x)$. { Note that we apply Lemma \ref{lemma:star} to obtain a map from the closure of $\tilde{V}_j$ to $B\left[0, \frac{\tau \pi}{2}\right]$, which we can then restrict to $\tilde{V}_j$ to obtain a bi-Lipschitz map from $\tilde{V}_j$ to $B(0, L)$}. We can define a new distribution $\tilde{\rho}_j$ on $B\left(0, \frac{\tau \pi}{2}\right)$ with density
\begin{align*}
    \tilde{p}_j(x) = \tilde{q}(F_j^{-1}(x))\left|\det J_{{F_j^{-1}}}(x)\right|.
\end{align*} It is important $F_j$ is bi-Lipschitz with constant $C$ as then $\frac{1}{C^d} \leq |\det J_{F_j}(x)| \leq C^d$ and we know $\tilde{p}_j$ is upper and lower-bounded. Further we can write $(F^{-1}
_j)_\sharp \tilde{\rho}_j = \tilde{Q}_j$ by $g_j^*$.

\textbf{Step 4: Constructing the local transport}. Now we have a distribution $\tilde{\rho}_j$ with density $\tilde{p}_j$ defined on convex ball $B\left(0, \frac{\tau \pi}{2} \right)$. Further $\tilde{p}_j$ is upper and lower bounded. Set $\rho_d$ to be uniform on $(0,1)^d$. Hence we may apply Proposition \ref{proposition:OT} to produce a transport map $T_j: (0,1)^d \to B\left(0, \frac{\tau \pi}{2} \right)$ such that $(T_j)_\sharp \rho_d = \tilde{\rho}$. Further $T_j$ is H\"older with H\"older exponent $\alpha_j$ for some $\alpha_j \in (0, 1)$. Writing $g_j^* : (0,1)^d \to V_j$ via  
\begin{equation}
g_j^*(x) = \exp_{c_j} \circ F_j^{-1} \circ T_j (x)
\label{eqgjstar}
\end{equation} gives rise to $(g_j^*)_\sharp \rho_d = Q_j.$ This transport map $g_j^*$ is illustrated in Figure \ref{fig:flowchart}.
Further $g_j^*$ is the composition of two Lipschitz functions with an $\alpha_j$-H\"older continuous function and is therefore $\alpha_j$-H\"older continuous. Thus we can push a uniform distribution on $(0,1)^d$ onto a local patch $V_j$, as shown in Figure \ref{fig:flowchart}.

\begin{figure}
    \centering
    \includegraphics[height=1.2in]{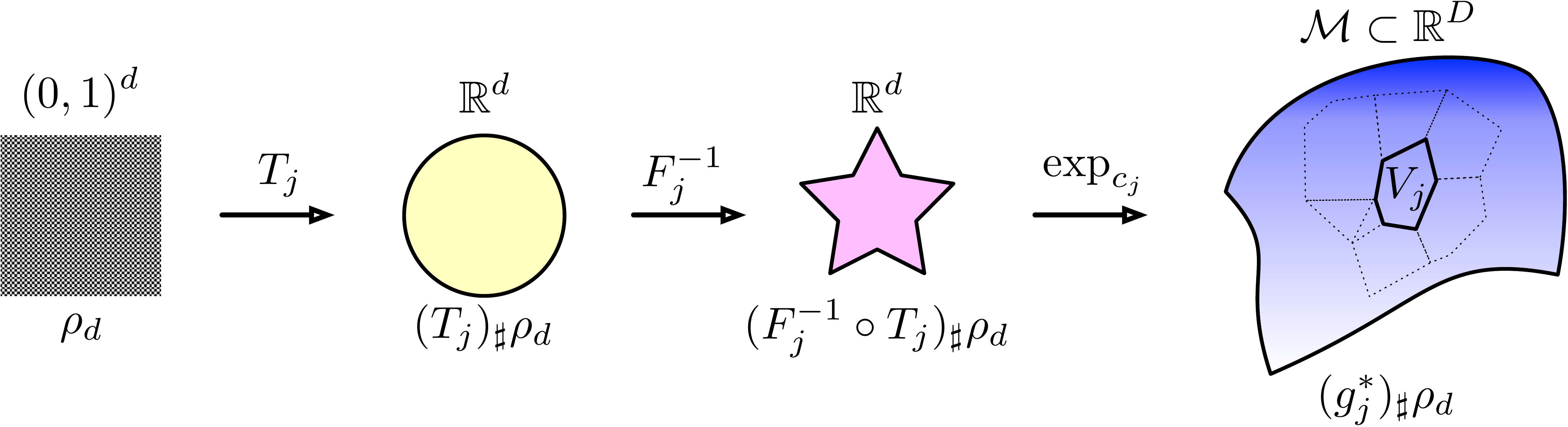}
    \caption{Local transport $g_j^*$ in \eqref{eqgjstar} mapping $\rho_d$ on $(0,1)^d$ to a local distribution $\Mm$ supported on $V_j$.}
    \label{fig:flowchart}
\end{figure}

\textbf{Step 5: Constructing the global transport}. Via step 4 we have local transports $g_j^*$ pushing forward $\rho_d$ uniform on $(0,1)^d$ to local distribution $Q_j$. By setting $\pi_j = Q(\bigcup_{i=1}^jV_i)$. Now we define $g^*: (0,1)^{d+1} \to \Mm$ via \eqref{eqgstar}

\begin{align*}
    g^*(x) = \sum_{j=1}^J\indicator_{[\pi_{j-1},\pi_j)}(x_1)g^*_j(x_{2:d+1})
\end{align*}

where $x_1$ is the first component of $x$ and $x_{2:d+1}$ are the remaining components. 

Set $X \sim \rho$. Then $X_1$ randomly selects the patch $V_j$ with probability equal to $Q(V_j)$, and $g_j^*$ pushes $X_{2:d+1}$ forward to $V_j$ such that $g_j^*(X_{2:d+1})$ is distributed according to $Q_j$. Finally note that the overlap of of the Voronoi cells $V_i$ and $V_j$ while $i\neq j$ has measure $0$. Hence $g^*(X)$ is distributed on $\Mm$ according to $Q$. This constructs the oracle $g^*$. Furthermore, we note that each of the $g_j^*$ is $\alpha_j$-H\"older continuous, so by setting $\alpha = \min_j \alpha_j$ we obtain that all of the $g_j^*$ are $\alpha$-H\"older continuous, which complete the proof of Lemma \ref{lemma:oracleappx}.

\subsection{Proof of Lemma \ref{lemma:star}}
\label{proof:star}

\begin{proof}
We will construct a bi-Lipschitz map from a compact star-shaped set centered at $0$ to a closed ball. Define the radial map $r : \mathbb{S}^{n-1} \to \R$ via $r(u) = \sup\{\lambda : \lambda u \in S\}$. Along a direction this gives us the radius of the set. Then via \citet{10.2307/2316022} this map is Lipschitz with optimal constant $C_1 = L\sqrt{\left(\frac{L}{\delta}\right)^2-1}$. Write $R(x) = r\left(\frac{x}{\|x\|}\right)$. We use this to define a function $A(x)$ on $S\backslash B(0,\delta/2)$ such that $A(x) = x$ if $\|x\|=\frac{\delta}{2}$ and $A(x) = L\frac{x}{\|x\|}$ if $\|x\| = R(x)$. Set $A: S \backslash B\left(0,\frac{\delta}{2}\right) \to B(0,L)$ via 
\begin{align*}
    A(x) = \left(\frac{\delta}{2}+\frac{L-\frac{\delta}{2}}{R(x)-\frac{\delta}{2}}\left(\|x\|-\frac{\delta}{2}\right)\right)\frac{x}{\|x\|}.
\end{align*}

We observe some properties of $A$. First, $A$ expands along rays meaning we can write $A(x) = \lambda_x x$ for some $\lambda$ dependent on $x$. Next, $A$ satisfies $A(x) = x$ when $\|x\|=\frac{\delta}{2}$ and $A(x)=L\frac{x}{\|x\|}$ when $\|x\|=R(x)$. Also, on a direction $u \in \mathbb{S}^{n-1}$ we have $A$ is strictly increasing in norm and hence injective on each ray. Thus $A$ is injective in total. Further, $A$ is surjective via an application of intermediate value theorem and we can write the inverse explicitly as 
\begin{equation*}
    A^{-1}(x) = \left(\frac{\delta}{2}+\frac{R(x)-\frac{\delta}{2}}{L-\frac{\delta}{2}}\left(\|x\|-\frac{\delta}{2}\right)\right)\frac{x}{\|x\|}.
\end{equation*} We verify this is the inverse by choosing $y = A(x)$ for some $x$ in the domain. Then
\begin{align*}
    \|y\| = \frac{\delta}{2}+\frac{R(x)-\frac{\delta}{2}}{L-\frac{\delta}{2}}\left(\|x\|-\frac{\delta}{2}\right).
\end{align*} 

Write $z = A^{-1}(y)$ and compute the norm as 
\begin{align*}
    \|z\| &= \left(\frac{\delta}{2}+\frac{R(x)-\frac{\delta}{2}}{L-\frac{\delta}{2}}\left(\|y\|-\frac{\delta}{2}\right)\right) \\
    &= \left(\frac{\delta}{2}+\frac{R(x)-\frac{\delta}{2}}{L-\frac{\delta}{2}}\left(\left(\frac{\delta}{2}+\frac{L-\frac{\delta}{2}}{R(x)-\frac{\delta}{2}}\left(\|x\|-\frac{\delta}{2}\right)\right)-\frac{\delta}{2}\right)\right) \\
    &= \left(\frac{\delta}{2}+\frac{R(x)-\frac{\delta}{2}}{L-\frac{\delta}{2}}\left(\frac{L-\frac{\delta}{2}}{R(x)-\frac{\delta}{2}}\left(\|x\|-\frac{\delta}{2}\right)\right)\right)\\
    &= \left(\frac{\delta}{2}+\frac{R(x)-\frac{\delta}{2}}{L-\frac{\delta}{2}}\left(\frac{L-\frac{\delta}{2}}{R(x)-\frac{\delta}{2}}\left(\|x\|-\frac{\delta}{2}\right)\right)\right)\\
    &= \frac{\delta}{2}+\|x\|-\frac{\delta}{2} = \|x\|.
\end{align*} Since $A$ expands along rays we have $z = x$ as desired.

 We can then define the functions $F_1: S \backslash B\left(0,\frac{\delta}{2}\right) \to B[0,L]$ and $F_2: B\left[0,\frac{\delta}{2}\right] \to B[0,L]$ via  $F_1(x) = A(x)$ and $F_2(x) = x$. Note immediately $F_2$ is bi-Lipschitz with constant $1$. 

Further we argue $A(x)$ is bi-Lipschitz. As a necessary warm-up we show $R(x) = r\left(\frac{x}{\|x\|}\right)$ is Lipschitz. We know $x \rightarrow \|x\|$ is Lipschitz with constant $1$, and $\|x\| \rightarrow \frac{1}{\|x\|}$ on $\left[\frac{\delta}{2},\infty\right)$ is Lipschitz away from $0$ with constant $\frac{4}{\delta^2}$, and $r(\cdot)$ is Lipschitz with constant $C_1$ given above. Now for $x,y \in S$ we have
\begin{align*}
    |R(x)-R(y)| &= \left|r\left(\frac{x}{\|x\|}\right)-r\left(\frac{y}{\|y\|}\right)\right|\\
    &\leq C_1\left\|\frac{x}{\|x\|}-\frac{y}{\|y\|}\right\| \\
    &\leq C_1\left(\left\|\frac{x}{\|x\|}-\frac{x}{\|y\|}\right\|+\left\|\frac{x}{\|y\|}-\frac{y}{\|y\|}\right\|\right) \\
    &\leq C_1\left(L\left|\frac{1}{\|x\|}-\frac{1}{\|y\|}\right|+\frac{2}{\delta}\|x-y\|\right)\\
    &\leq C_1\left(L\frac{4}{\delta^2}\|x-y\|+\frac{2}{\delta}\|x-y\|\right) \\
    &\leq C_1\left(\frac{4L}{\delta^2}+\frac{2}{\delta}\right)\|x-y\|
\end{align*}
 as desired. Notice it is important that the set $S$ is bounded and the operator $R$ is defined at a radius $\frac{\delta}{2}$ away from $0$. Using this we can show $A$ is Lipschitz. Clearly $\|x\| -\frac{\delta}{2}$ is Lipschitz and bounded since $S$ is bounded. $\frac{1}{R(x)-\frac{\delta}{2}}$ will be Lipschitz as $R(x)-\frac{\delta}{2}$ is Lipschitz and we have $R(x) \geq \delta$ so the denominator is lower bounded. Thus $\frac{L-\frac{\delta}{2}}{R(x)-\frac{\delta}{2}}(\|x\|-\frac{\delta}{2})$ is Lipschitz as the product of bounded Lipschitz functions. It follows $A$ is Lipschitz, as the product of bounded Lipschitz functions, with some constant $C_2$. A nearly identical argument shows $A^{-1}$ is Lipschitz. Hence $F_1$ is bi-Lipschitz.

Define a function $F$ on $S$ as $F_1$ on $S\backslash B(0,\delta/2)$ and $F_2$ on $B(0,\delta/2)$. This will be Lipschitz as $F$ is Lipschitz on each of the two domains and for $x \in S\backslash B(0,\delta/2), y \in B(0,\delta/2)$ we have 
\begin{align*}
    |F(x) - F(y)| &\leq |F(x)-F(z)|+|F(z)-F(y)| \leq  C_2\|x-z\|+\|z-y\| \\
    &\leq C_2\|x-y\| + \|x-y\| = (C_2+1)\|x-y\|
\end{align*} 

where we pick $z \in [x:y]\cap B(0,\delta/2)$ and recall $[x:y]$ is the line between $x$ and $y$. Hence $|x-z|,|y-z| \leq |x-y|$. This shows $F$ Lipschitz. It is also clearly bijective as each of its components are bijective and we can write $F^{-1}$ explicitly. A similar argument shows $F^{-1}$ is Lipschitz. Thus $F$ is bi-Lipschitz.
\end{proof}

\subsection{Voronoi cells have measure 0 overlap}
\begin{proof}
    We have that $Q(V_i \cap V_j) = 0$. To see this, it suffices to show the set of points $I_{ij} = \{x \in U_i \cap U_j : d_{\Mm}(x,c_i) = d_{\Mm}(x,c_j)\}$ is a codim 1 submanifold of $M$, which means it has Hausdorff measure 0. Note that $I_{ij} = f_{ij}^{-1}(0)$ where $f_{ij} : U_i \cap U_j \rightarrow \R : x \mapsto d_\Mm^2(c_i, x) - d_\Mm^2(c_j, x)$, so it suffices to show that $0$ is a regular value of $f_{ij}$ which is smooth in $U_i \cap U_j$. Note that $df_x = -2(\exp_x^{-1}(c_i) - \exp_x^{-1}(c_j))$. For all $x \in f^{-1}_{ij}(0)$, we need to show that $df_x$ has full rank. Since the co-domain of $df_x$ is one dimensional, this is equivalent to showing that $df_x \neq 0$. Suppose for the sake of contradiction that $df_x = 0$. Then this implies $\exp_x^{-1}(c_i) = \exp_x^{-1}(c_j)$. Since $\exp_x$ is a diffeomorphism in a neighborhood that includes $c_i$ and $c_j$, this implies that $c_i = c_j$. But this is not true for $i \neq j$. Thus $df_x$ is full rank for all $x \in U_i \cap U_j$ such that $f(x) = 0$, so $0$ is a regular value of $f$.
\end{proof}

\end{appendix}

\end{document}